\documentclass[manuscript, screen, review]{jair}

\setcopyright{cc}
\RequirePackage[
  datamodel=acmdatamodel,
  style=acmauthoryear,
  backend=biber,
  giveninits=true,
  uniquename=init
  ]{biblatex}

\usepackage{enumitem}
\usepackage{amsthm}

\newtheorem{definition}{Definition}
\newtheorem{remark}{Remark}
\usepackage{newunicodechar}
\renewcommand\footnotetextcopyrightpermission[1]{}

\begin{document}

%\title{JAIR Example Template}
\title{The Dynamic Gist-Based Memory Model (DGMM):  A Memory-Centric Architecture for Artificial Intelligence}

%% used to denote shared contribution to the research and/or corresponding author.
\author{Terry Dorsey}
\authornote{Corresponding Author}
\orcid{0009-0000-9405-4795}
\email{tdorsey1@harrisburgu.edu}
\affiliation{%
  \institution{Harrisburg University of Science and Technology}
  \city{Hanover}
  \state{Pennsylvania}
  \country{USA}
}

\author{Kevin Huggins}
\authornote{Senior Author}
\orcid{0009-0008-3989-4677}
\email{khuggins@harrisburgu.edu}
\affiliation{%
  \institution{Harrisburg University of Science and Technology}
  \city{Hanover}
  \state{Pennsylvania}
  \country{USA}
}

%% The short list of authors must be made of the list of all authors' lastnames.

\renewcommand{\shortauthors}{Dorsey}
%% If this is too long and overlaps other information printed in the page headers, use
%\renewcommand{\shortauthors}{Xu et al.}

\begin{abstract}
{\bf Background:} 
Contemporary artificial intelligence systems achieve strong performance through large-scale parameterization, retrieval augmentation, and training on extensive static corpora. Despite these advances, such systems continue to face limitations in persistent memory, temporal grounding, provenance, and interpretability. These challenges are particularly pronounced in large language models, where experience is encoded implicitly in fixed parameters, constraining the ability to preserve, inspect, and reinterpret past interactions over time.

{\bf Objectives:}
This paper aims to establish a memory-centric architectural foundation for artificial intelligence in which experience is represented explicitly, persistently, and in a form that supports temporal grounding, provenance, and interpretability. The goal is to define an alternative to parameter-centric approaches by treating memory as a first-class, structured substrate for reasoning.

{\bf Methods:}
The paper introduces the Dynamic Gist-Based Memory Model (DGMM), an architectural framework in which experience is represented as an evolving, graph-structured episodic–semantic memory. DGMM encodes experience as interconnected conceptual structures grounded in time, source, and interaction context, and defines selective, cue-conditioned recall as the mechanism through which working memory is constructed. A formal characterization of the representational schema is provided, along with a set of architectural invariants derived from commitments to additive memory growth and recall-conditioned interpretation.

{\bf Results:}
The primary results consist of formally specified properties of the DGMM architecture, including episodic persistence, locality of cue-conditioned surprise, and contextual variability without structural modification of stored memory. These properties establish internal consistency and define constraints on admissible memory operations and recall behavior.

{\bf Conclusions:} 
DGMM provides a coherent architectural theory in which memory is explicit, persistent, and structurally grounded. By decoupling memory storage from downstream interpretation, the framework supports evolving perspectives without retraining or modification of stored representations. This work establishes a foundation for future research into interpretable, context-aware, and temporally grounded artificial intelligence systems.
\end{abstract}

%\begin{abstract}
%      A clear and well-documented \LaTeX\ document is presented as an
%  article formatted for publication by ACM in a conference proceedings
%  or journal publication. Based on the ``acmart'' document class, this
%  article presents and explains many of the common variations, as well
%  as many of the formatting elements an author may use in the
%  preparation of the documentation of their work.
%\end{abstract}

%% JAIR Note: 
%% Do not include ACM CCS Concepts or Keywords

%% To be updated by authors.
%\received{20 May 2026}
%\received[accepted]{31 Dec 2027}

\maketitle
\makeatletter
\fancyhf{}
\fancyfoot[L]{\today}
\fancyfoot[C]{\thepage}
\pagestyle{fancy}
\thispagestyle{fancy}
\makeatother

\section{Introduction}
The relationship between memory and intelligence is not incidental.  To reason about the past, anticipate the future, or maintain coherent understanding across time, a cognitive system must be able to preserve experience, retrieve it selectively, and reinterpret it as context and perspective evolve.  These capacities are not merely useful features of an intelligent system, they are architectural preconditions for the kind of temporal, contextual, and provenance-aware reasoning that distinguishes genuine understanding from pattern recognition.

Contemporary artificial intelligence systems have achieved remarkable performance without satisfying these preconditions in any deep sense.  Large language models encode experience implicitly in fixed parameters learned from static corpora, producing systems that are powerful but architecturally amnesiac.  Each inference begins from the same parametric state, with no persistent record of what the system has encountered, when it encountered it, or from whom.  Retrieval-augmented approaches extend this by introducing access to external data, but retrieved content is assembled transiently at inference time and discarded afterward, leaving no evolving memory substrate that accumulates, develops, or can be examined across interactions.  The result is a class of systems that exhibit the outputs of intelligence without the underlying continuity of experience that makes those outputs interpretable, trustworthy, or genuinely grounded.  

These are not incidental engineering limitations.  They reflect a deeper architectural choice: to treat memory as implicit, transient, or auxiliary rather than as a first-class component of the system's design.  The consequences of this choice (e.g., hallucination, temporal misalignment, semantic drift, unstable recall) have been extensively documented and are widely attributed to failures of training data, optimization objectives, or model scale.  This paper advances a different diagnosis: these phenomena are symptomatic of architectures that were never designed to remember.

This paper advances the position and importance of treating memory as a first-class architectural component in artificial intelligence systems, with interpretation deferred until recall.  We introduce the Dynamic Gist-Based Memory Model (DGMM), a memory-centric architecture in which experiences are stored as explicit, persistent structures grounded in time, source, and interaction, independent of task- or language-specific encoding.  By decoupling memory storage from downstream interpretation, DGMM allows recalled experience to be reinterpreted as language, context, or analytical needs evolve and do so without training.

The contributions of this work are architectural and theoretical.  First, it re-frames widely observed limitations in contemporary AI systems as consequences of short-horizon reasoning rooted in architectural design rather than model capability.  Second, it formalizes a memory-centric architectural approach in which experience is preserved independently of downstream tasks or representations.  Third, it articulates post-recall processing as a core architectural principle that allows stable memory to support evolving interpretation over time.  Accordingly, DGMM is presented not as a deployed system or algorithmic solution, but as an architectural theory intended to clarify how memory representation constrains and enables downstream behavior.   Subsequent work will examine empirical instantiations, recall strategies, and comparative evaluation across memory-intensive tasks.

This paper is the first in a series of planned contributions developing the DGMM framework.  The present work establishes the architectural foundation: the formal schema, the memory operation regime partition, and the axiomatic definitions that subsequent papers will build on directly.

\section{Background and Motivation}

Recent interdisciplinary research increasingly characterizes memory as dynamic, reconstructive, and relational, shaped by continuous interaction between episodic experience and semantic knowledge rather than by static storage of factual representations
\cite{dalessandro_genesis_2025,fayyaz_model_2022,jones_models_2015,kumar_semantic_2021,nagy_interplay_2025,spens_generative_2024,spens_hippocampo-neocortical_2025}.  These findings motivate a shift in artificial intelligence away from treating memory as an auxiliary or implicit component and toward treating it as a first-class architectural component that preserves experience over time.

This perspective reframes persistent limitations observed in contemporary artificial intelligence systems, particularly Large Language Models (LLMs).  Despite their success, LLMs exhibit recurring phenomena such as hallucination (Xu et al., 2024; J.-Y. Yao et al., 2023), temporal misalignment \cite{dhingra_time-aware_2022,luu_time_2022,nylund_time_2023,zhao_set_2024}, semantic drift  \cite{kong_dynamic_2024,yang_beyond_2025} and unstable recall \cite{duan_few-shot_2023,jiang_wikipedia-based_2017,kong_dynamic_2024,zhang_respecting_2025} While often attributed to training data quality \cite{bender_dangers_2021,tonmoy_comprehensive_2024}, optimization objectives \cite{xu_hallucination_2024}, or model scale \cite{bender_dangers_2021,brown_language_2020}, converging evidence suggests these behaviors reflect deeper architectural deficiencies in how memory is represented, persisted, and updated \cite{zhao_set_2024}.

From a memory perspective, these deficiencies arise because LLMs encode knowledge implicitly in parameters \cite{ferrario_how_2022,luo_local_2024,wu_continual_2024}, lack long-term continuity beyond bounded context windows \cite{brown_language_2020,jiang_wikipedia-based_2017,kong_dynamic_2024,zhang_survey_2025,zhang_respecting_2025}, require retraining to integrate new information \cite{lewis_retrieval-augmented_2020,wu_continual_2024,yao_tree_2023,zhao_survey_2023}, and do not maintain explicit representations of provenance, temporal grounding, or relational structure.  Retrieval-augmented and agentic memory approaches partially address grounding and continuity, but typically treat memory as an external attachment reintroduced as text or features rather than as a persistent representation of experience \cite{anokhin_arigraph_2024,bernar_exploring_2024,dakat_enhancing_2024,das_larimar_2024,gai_zhenhua_achieving_2024,gutierrez_hipporag_2024,miao_episodic_2024,rasmussen_zep_2025}.

Knowledge graphs offer explicit relational structure and improve grounding and interpretability when paired with language models \cite{lewis_retrieval-augmented_2020,peng_knowledge_2023,sahoo_systematic_2024}, yet they remain largely static and fact-centric.  Temporal extensions often treat time as metadata or embedding dynamics rather than a foundational structural element capable of supporting selective recall, perspective diversity, and semantic evolutionn \cite{biswas_knowledge_2023,cai_survey_2024,knez_event-centric_2023,liu_incremental_2022,sadeghian_chronor_2021,saxena_question_2021,sun_timelinekgqa_2025,zhang_respecting_2025}.  Collectively, these limitations motivate the need for a memory-centric architecture that preserves experience as an explicit, evolving structure rather than as implicit parameters or transient context.

The challenge of preserving prior knowledge while integrating new experience without retraining has been studied extensively in the continual learning literature under the rubric of catastrophic forgetting and the stability-plasticity tradeoff \cite{kirkpatrick_overcoming_2017,lange_continual_2021,thrun_child_1998,thrun_lifelong_1995,van_de_ven_three_2022,wang_comprehensive_2024}.  Existing approaches address this challenge primarily through parameter management (e.g. regularization, architectural expansion, or rehearsal) operating on the assumption that knowledge is encoded implicitly in weights.  DGMM responds to the same motivating problem from a different architectural premise by representing memory explicitly as a persistent structure that grows additively whereby the stability-plasticity tradeoff is addressed at the level of memory architecture rather than parameter optimization.

\section{Cognitive Foundations}
Cognitive psychology and neuroscience provide clearer guidance on the structural properties required of a memory system.  Biological memory is organized in distributed cell assemblies, or engrams, characterized by sparse global activation and dense local connectivity \cite{bullmore_economy_2012,hebb_donald_1949}.  Such organization supports scalable storage, preserves the integrity of past experiences against interference from new learning, and enables selective reinstatement.

Recall in biological systems is cue-dependent and selective, with partial reinstatement of temporal, perceptual, or contextual cues sufficient to reactivate prior memory states and support pattern completion \cite{howard_distributed_2002}.  Temporal information is preserved through mechanisms that encode order, proximity, and contextual continuity across experiences \cite{schonhaut_neural_2023,umbach_time_2020}.  Memory content further evolves as new information is integrated, producing gradual shifts in meaning rather than abrupt overwriting \cite{poo_what_2016}.  Together, these findings imply that memory architecture must support structural persistence, selective access, temporal coherence, and controlled semantic evolution conditioned on cue-driven and contextual reinstatement.

Research on episodic and semantic memory reinforces these requirements.  While early models treated semantic memory as a static repository distinct from episodic experience, contemporary accounts demonstrate that the two systems are deeply intertwined \cite{dalessandro_genesis_2025,fayyaz_model_2022,kumar_semantic_2021,nagy_interplay_2025,spens_hippocampo-neocortical_2025}.  Generative and rate-distortion models show that episodic traces store compressed, gist-like representations that are reconstructed during recall using semantic priors, producing systematic distortions that favor predictive relevance over literal accuracy \cite{fayyaz_model_2022,nagy_interplay_2025,spens_hippocampo-neocortical_2025}.  \cite{kumar_semantic_2021} specifically identifies four unresolved gaps in contemporary semantic models: they cannot track perspective diversity across contexts; they conflate experiences from different sources into a single, averaged representation; they lack mechanisms for maintaining provenance or viewpoint and they cannot represent how meaning evolves over time.  These findings motivate memory architectures that support reconstructive recall grounded in evolving semantic structure.

These architectural commitments reflect a broader philosophical position on the nature of meaning and experience.  Following Kant's insight that experience must be actively structured to be coherent, DGMM's fixed relational grammar establishes the categories through which experience can be represented, not as an arbitrary data model but as a precondition for coherent episodic memory \cite{kant_critique_1781}.  Hegel's argues that meaning develops historically through reinterpretation motivates the separation of storage from interpretation: what is preserved is not a fixed meaning but the structured trace from which meaning can be reconstructed as context and perspective evolve \cite{hegel_georg_1807}.  Foucault's account of knowledge as provenance-conditioned and irreducibly plural motivates the Source node architecture and the principle that conflicting accounts should coexist in memory rather than be reconciled at storage time \cite{foucault_archaeology_1972}.  The Derrida's concept of différance, that meaning as constitutively deferred rather than present, provides the philosophical grounding for late semantic commitment as an architectural principle: meaning is not stored in DGMM because it cannot be stored; it is produced at recall time through the differential activation of memory structure in response to a cue \cite{derrida_grammatology_1998}.

Together, these cognitive and philosophical foundations identify essential properties a memory architecture must embody: engram-like structural organization, cue-dependent recall, temporal coherence, reconstructive memory, and evolving meaning.  These principles provide the theoretical grounding for DGMM’s design as an explicit, graph-based episodic–semantic memory substrate.

\section{Architectural Design Principles}
\label{sec:architecture}
DGMM is built around explicit structural principles that prioritize interpretability, consistency, and incremental adaptation.  Memory is represented using a fixed schema of typed nodes and relations, separating conceptual substance from contextual modification.  New experiences extend the memory structure without overwriting prior representations, reflecting a stability–plasticity balance analogous to biological memory systems.

Crucially, DGMM separates memory storage from interpretation.  Stored representations capture the gist of experience rather than surface linguistic form, enabling recalled memory to be reinterpreted under new contexts without modifying the underlying structure.

\subsubsection{The Dynamic Gist-Based Memory Model (DGMM)}
\label{sec:dgmm}
This section formalizes the Dynamic Gist-Based Memory Model (DGMM), a memory-centric architecture that separates episodic storage, working memory construction, and semantic interpretation into distinct but structurally coupled components.  DGMM is characterized by a fixed relational grammar, additive real-time memory growth, and semantic meaning derived through cue-conditioned sensemaking rather than modification of stored memory.

\subsubsection{Long-Term Memory as a Fixed Relational Grammar}
\label{sec:fixedgrammer}
Long-term memory in DGMM is fundamentally episodic.  Every structure stored in memory originates in a specific experience, grounded in a particular time, source, and interaction context.  This episodic character constitutes how memory is represented.

A central limitation of contemporary AI memory systems is that they fix semantic content at the point of encoding, embedding meaning into parameters or representations that cannot be revisited from a different interpretive standpoint.  DGMM is designed to avoid this commitment.  To do so, memory must preserve experience in its specific, contextually grounded form rather than as pre-interpreted content.  A concept stored without its temporal and contextual origins has already been interpreted, reduced from specific experience to a general abstraction, and that reduction cannot be undone at recall time.  In DGMM, abstraction is therefore the work of recall and interpretation, not of storage. 

In DGMM, a gist is the structured conceptual representation of an experience, that is, its essential relational content expressed through the node and relation types of the fixed schema, independent of the surface linguistic or syntactic form in which the experience was originally encountered.  Storage operates at the level of gist rather than verbatim form, preserving what an experience means structurally while remaining neutral with respect to how it was expressed. 

Because memory is episodic and experiences accumulate over time, the memory structure is inherently dynamic.  At any point in time t, long-term memory reflects the totality of experience ingested up to that point.  DGMM represents this as a typed episodic graph whose state at time t is given by:

\[
M_t = (V_t, E_t, S)
\]
where $V_t$ is the set of nodes present at time $t$, 
$E_t \subseteq V_t \times V_t$ is the set of relations present at time $t$, 
and $S$ is a fixed schema that defines the admissible structure of memory. 
The schema $S$ is time-invariant and therefore defines what kinds of structures may exist in memory at any point, 
while $V_t$ and $E_t$ grow as DGMM ingests new experience. 
The memory state $M_t$ defines a snapshot of an evolving structure, 
and any property stated about memory is implicitly a property of $M_t$ at a given $t$.

Nodes in $V$ are instantiated from a closed set of node types $T_V$, 
each corresponding to a stable cognitive role, including \textit{Concept}, \textit{Element}, \textit{Time}, \textit{Interaction}, and \textit{Source}. 
Relations in $E$ are instantiated from a closed vocabulary of relation types $T_E$ 
(e.g., \textsc{HAS\_SUBJECT}, \textsc{HAS\_ACTION}, \textsc{HAS\_OBJECT}, \textsc{MODIFY\_SUBJECT}, \textsc{ACQUIRED\_AT}, \textsc{RECOUNTS}).

The schema $S$ further specifies type-pair admissibility constraints: 
for any ordered pair of node types $(\tau_i, \tau_j) \in T_V \times T_V$, 
only a fixed subset of relation types is permitted,
\begin{equation}
T_E^{(\tau_i, \tau_j)} \subseteq T_E
\end{equation}

enforcing a fixed relational grammar. 
This grammar encodes how experiences may be represented and prevents schema drift, 
arbitrary relation invention, or post-hoc reinterpretation of structure. 
Figure~\ref{fig:dgmm-schema} illustrates the fixed node and relation structure defined by this grammar.

\begin{figure}[ht]
  \centering
  \includegraphics[width=0.8\linewidth]{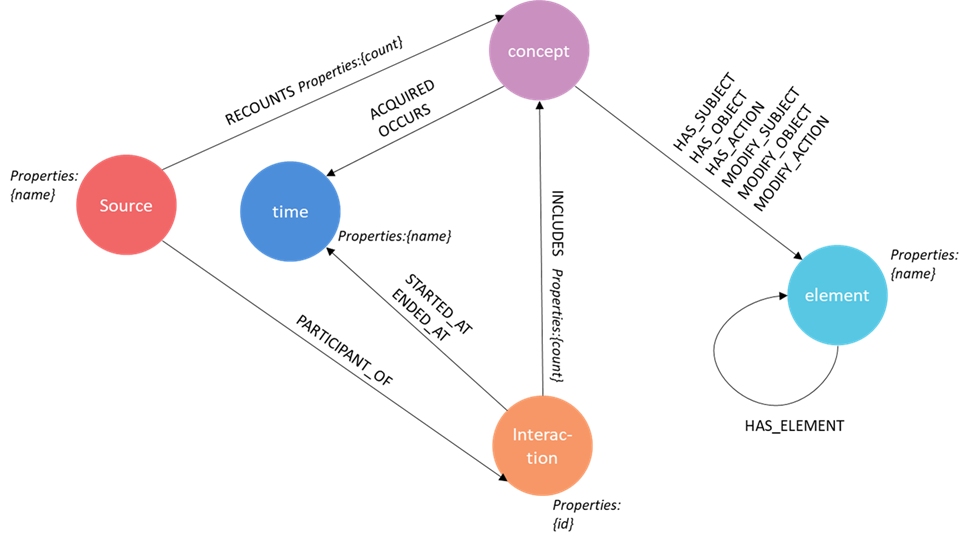}
  \caption{Fixed relational grammar of DGMM. Core node types—Concept, Element, Time, Interaction, and Source—and the admissible relations between them define how experience is represented persistently in memory while constraining how nodes may be connected.}
  \Description{DGMM schema showing node types Concept, Element, Time, Interaction, and Source with allowed relations between them.}
  \label{fig:dgmm-schema}
\end{figure}

Figure~\ref{fig:dgmm-example} illustrates the core representational structure of the Dynamic Gist-Based Memory Model (DGMM), centered on a Concept node that serves as the integrative anchor for an episodic memory instance. Surrounding this core, the model organizes information into distinct but interconnected contextual dimensions, each instantiated as nodes within the graph. On the left, a cluster of Elements captures the semantic composition of the episode, including entities (e.g., Jack, Jill), actions (e.g., fetch), objects (e.g., water), and modifiers (e.g., fresh), with labeled relationships such as \textsc{HAS\_SUBJECT} and \textsc{MODIFY\_OBJECT} representing their roles and interactions. On the right, the Time dimension situates the concept across multiple temporal references, distinguishing between the time of the original event or concept occurrence (e.g., 1/1/2000) and the time at which DGMM acquired or ingested that information into memory (e.g., 11/15/2025). The structure is grounded by a Source node, which ties the memory to its origin (e.g., \textit{Jack \& Jill Book}), and an Interaction identifier that uniquely distinguishes the memory instance. Additional contextual dimensions, such as Emotion and Space, are included to demonstrate the extensibility of the model, illustrating how new contextual layers can be incorporated without altering the underlying structural framework. Together, this graph-based representation encodes not just static facts, but the relational, temporal, and contextual dynamics that characterize DGMM’s episodic memory.

\begin{figure}[ht]
  \centering
  \includegraphics[width=0.8\linewidth]{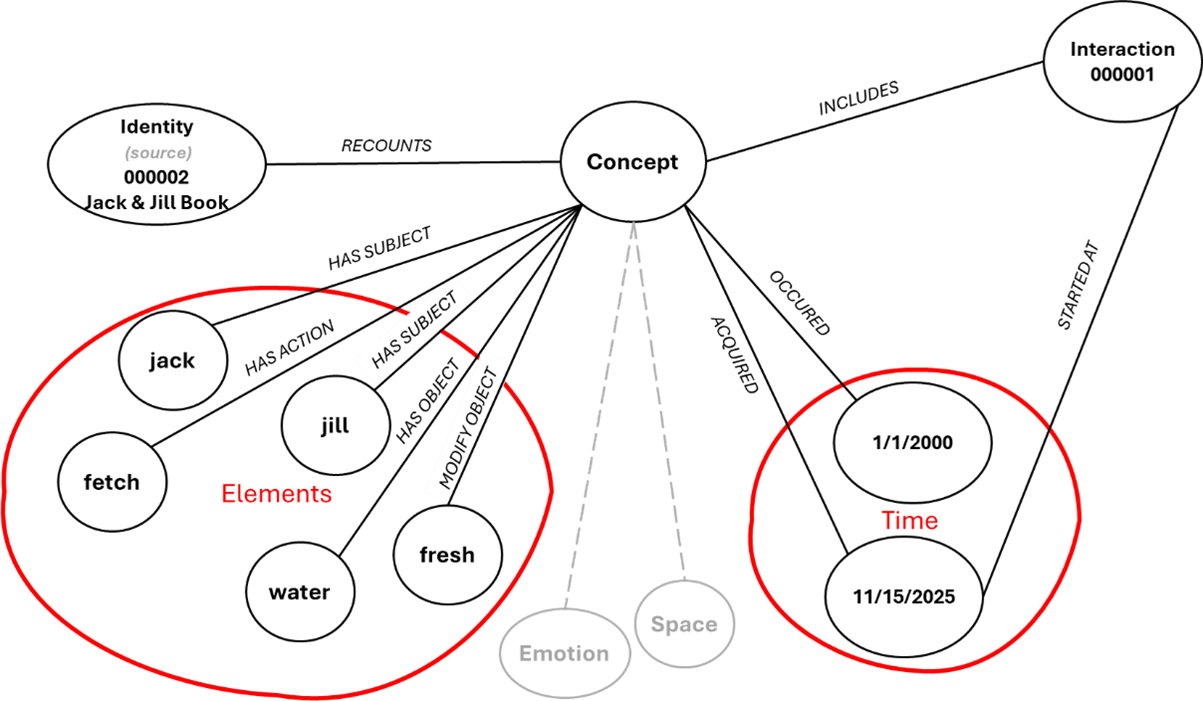}
  \caption{DGMM episodic memory instance. A Concept node anchors the memory, with associated Elements (subjects, actions, objects, modifiers), temporal references, source provenance, and interaction context forming a structured, multi-dimensional representation.}
  \Description{Example DGMM graph centered on a Concept node with Elements, Time nodes, Source, and Interaction context.}
  \label{fig:dgmm-example}
\end{figure}

\subsubsection{Memory Operation Regimes}
\label{memory-operation-regime}
DGMM distinguishes four operationally distinct regimes through which memory is formed, structured, accessed and examined.  These regimes differ in what triggers them, how they interact with the memory graph, and what invariants govern them.  Establishing this partition explicitly is essential because the guarantees available to downstream operations depend critically on which regime has acted on the memory state at any given time.

\begin{definition}[Memory Operation Regimes]
\label{def:memory-operation-regimes}
All operations on $M_t$ belong to exactly one of four regimes.

\paragraph{Ingestion.}
Ingestion is the process by which external experience is projected into the memory graph. 
Ingestion operations are strictly additive: they introduce new nodes and relations conforming to the fixed schema $S$ while leaving all existing nodes and relations unchanged. 
Ingestion is triggered by external input and is the sole mechanism through which new episodic content enters $M_t$. 
No ingestion operation may modify, retype, or remove an existing node or relation.

\paragraph{Consolidation.}
Consolidation is the process by which existing memory structure is internally reorganized. 
Consolidation operations are transformative: they may restructure, merge, or re-weight existing nodes and relations subject to two governing constraints. 
First, consolidation must be schema-preserving, in that all structures produced by consolidation must conform to $S$. 
Second, consolidation must be provenance-preserving. 
Time, Source, and Interaction nodes associated with any consolidated structure must be retained and remain accessible in the reorganized graph; however, DGMM provides an allowance for Time nodes to be generalized.

Consolidation is an internal process, decoupled from both ingestion and recall, and may be triggered by structural conditions within $M_t$ rather than by external input.

\paragraph{Recall.}
Recall is the process by which a subset of memory is selected and made available for interpretation in response to a cue. 
Recall operations are strictly read-only: they construct a working memory subgraph $W_q \subseteq M_t$ without modifying $M_t$ in any way. 
Recall is triggered by a cue $q$, and its output is transient in that working memory exists only for the duration of the interpretive operation it supports and is not written back to long-term memory.

\paragraph{Analysis.}
Analysis is the process by which recalled structure is interpreted, transformed, or evaluated to produce derived representations. 
Analytic operations take $W_q$ as input and produce outputs (e.g., embeddings, propositions, surprise signals, attribution distributions) without modifying either $M_t$ or $W_q$. 
Analysis is strictly post-recall: it presupposes a constructed working memory and has no direct access to long-term memory except through what recall has made available. 
Analytic outputs are transient unless explicitly committed to memory through a subsequent ingestion operation.

These four regimes are mutually exclusive with respect to their effects on $M_t$. 
Ingestion extends the graph, consolidation reorganizes it, recall reads from it, and analysis examines recall-induced subgraphs. 
No operation spans more than one regime.
\end{definition}

\begin{remark}[Regime Boundaries and the Series]
Each regime gives rise to distinct theoretical questions and distinct performance considerations. 
The ingestion regime and its additive growth invariant are developed formally in Section~\ref{sec:memory-ingestion}. 
The consolidation regime, its admissible policies, and the performance tradeoffs it entails are deferred for future research.  The recall regime and its axiomatic characterization are developed in Section~\ref{sec:selective-recall} and extensions are slated for development in future work.
\end{remark}

\begin{table}[ht]
\centering
\caption{DGMM Memory Operation Regimes}
\label{tab:dgmm-regimes}
\begin{tabular}{lllll}
\toprule
\textbf{Regime} & \textbf{Trigger} & \textbf{Input} & \textbf{Effect on $M_t$} & \textbf{Governing Constraint} \\
\midrule
Ingestion      & External input   & Experience & Additive only & Schema conformance \\
Consolidation  & Internal condition & $M_t$     & Transformative & Schema + provenance preservation \\
Recall         & Cue $q$          & $M_t$     & None           & Subgraph + boundedness \\
Analysis       & Post-recall      & $W_q$     & None           & Read-only on $W_q$ and $M_t$ \\
\bottomrule
\end{tabular}
\end{table}

\subsubsection{Identity and Uniqueness Constraints}
\label{sec:identity-constraint}
DGMM enforces explicit identity semantics across node types. 
Element and Time nodes are subject to name-based uniqueness and are reused across memory ingestion:
\begin{equation}
\forall v_i, v_j \in V_{\text{Elem}} \cup V_{\text{Time}}, \;
\text{name}(v_i) = \text{name}(v_j) \;\Rightarrow\; v_i = v_j
\end{equation}

These nodes function as canonical anchors for perceptual grounding and temporal alignment.

In contrast, Concept nodes are instance-specific and are not subject to name-based uniqueness. 
Distinct Concept nodes may share the same label while maintaining separate identities and relational neighborhoods, 
allowing semantic meaning to remain contextualized rather than globally collapsed.

\subsubsection{Interaction and Source Nodes}
\label{sec:interaction-source-nodes}
DGMM enforces episodic structure and provenance through Interaction and Source nodes.

Interaction nodes represent discrete episodes or sessions and group Concept nodes that participate in a single interactional context.  They establish explicit episodic boundaries, preventing the collapse of experiences into undifferentiated semantic aggregates.

Source nodes encode provenance and perspective, representing agents, systems, documents, or sensors from which information originates.  Concepts and interactions may be linked to one or more Source nodes, enabling traceability and supporting multi-source or conflicting accounts without reconciliation at storage time.

\subsection{Memory Operations}
\label{sec:memory-operations}
DGMM defines four core memory operations that describe how experience is represented, accessed, and allowed to evolve within the architecture.  These operations specify what structural transformations are projected to occur, rather than prescribing algorithms or implementation strategies.

\subsubsection{Memory Ingestion}
\label{sec:memory-ingestion}
Memory formation via ingestion in DGMM refers to the architectural process by which incoming experience is projected into persistent memory structures. 
Rather than storing raw inputs or surface linguistic forms, DGMM is designed to accept experience decomposed into gist-level conceptual structures. 
Each stored concept is explicitly associated with its constituent elements, as well as with time, source, and interaction context.

Architecturally, memory ingestion results in the addition of new nodes and relations to the memory graph, extending existing structures without overwriting prior representations. 
This ensures that persistence, provenance, and temporal grounding are properties of the memory substrate itself rather than emergent effects of training or context management.

While the schema $S$ is fixed, the memory instance $M_t$ is extensible. As new experiences occur, nodes and relations may be appended in real time, provided they conform to the fixed relational grammar and identity constraints.

Memory ingestion growth is additive rather than mutative: existing nodes and relations are not retyped, restructured, or overwritten. The additive character of ingestion is formalized as part of the architectural invariants in Section~\ref{sec:structural-persistence}, where it is shown to follow necessarily from the regime partition established in Section~\ref{sec:memory-ingestion}.

\subsubsection{Selective Recall}
\label{sec:selective-recall}
Selective recall describes the architectural mechanism by which a subset of stored memory is projected back into an active working context in response to a cue.  DGMM is designed to support recall that is conditioned on conceptual, temporal, source, or interaction-level constraints, enabling partial reinstatement of memory in lieu of exhaustive retrieval.

From a structural perspective, selective recall corresponds to the identification and extraction of a bounded subgraph whose structure reflects relevance to the cue.  This recalled subgraph functions as an interface between persistent memory and downstream processing, supporting contextual continuity while preserving separation between storage and interpretation.

\begin{definition}[Working Memory Construction (Recall)]
\label{def:working-memory-construction}
Given a cue $q$ and memory state $M_t$, DGMM constructs a working memory state $W_q$ defined as:
\begin{equation}
W_q = \psi(q, M_t), \quad W_q \subseteq M_t
\end{equation}
\end{definition}

The operator $\psi$ is the recall operator. Its formal properties are established in Section~\ref{sec:admissible-recall-operators} (Definition~\ref{def:admissible-recall-operator}). 
Working memory $W_q$ is transient and therefore recomputed per cue, functioning as a bounded cognitive workspace that maintains separation between storage and interpretation, protecting long-term memory from interference.

\paragraph{Admissible Recall Operators}
\label{sec:admissible-recall-operators}
The recall operator $\psi$ is not prescribed algorithmically in this work. Specific recall strategies, their structural properties, and their performance characteristics under varying memory conditions are developed in a companion paper. The role of this section is to establish the axiomatic constraints that any compliant recall operator must satisfy, defining the architectural contract within which recall strategies operate and providing the formal foundation from which the companion paper proceeds.

\begin{definition}[Admissible Recall Operator]
\label{def:admissible-recall-operator}
An operator $\psi : Q \times M_t \to W_q$ is admissible under DGMM if and only if it satisfies the following axioms:

\begin{enumerate}[label=\textbf{R\arabic*:}, leftmargin=*]

\item \textbf{Subgraph.}
$W_q \subseteq M_t$ for all $q$ and $t$. Recall constructs a subgraph of long-term memory and does not introduce structure absent from $M_t$.

\item \textbf{Non-triviality.}
For any non-empty $M_t$ and cue $q$ that is present in memory, for which at least one node or subgraph in $M_t$ satisfies the cue condition, $W_q \neq \emptyset$. Recall returns at least some content when the cue has a structural correlate in long-term memory.

\item \textbf{Cue-sensitivity.}
There exist cues $q_1 \neq q_2$ such that
\[
\psi(q_1, M_t) \neq \psi(q_2, M_t).
\]
Recall is genuinely conditioned on the cue rather than returning a fixed structure independent of it.

\item \textbf{Recall Boundedness.}
$|W_q| < |M_t|$ in general. Working memory is a proper subset of long-term memory; recall is selective rather than exhaustive.

\item \textbf{Read-only Recall.}
$M_t$ is unchanged by the application of $\psi$. Recall operates within the read-only regime established in Section~\ref{sec:memory-operations} and does not modify long-term memory.

\end{enumerate}
\end{definition}

These axioms define the space of compliant recall functions rather than a unique function. Multiple recall strategies may satisfy these constraints while differing substantially in what they retrieve, how efficiently they operate, and what performance characteristics they exhibit under varying memory conditions. This multiplicity is intentional: DGMM is an architectural theory, and the choice among admissible recall strategies is a design and performance question addressed in the companion paper.

\begin{remark}[Recall Classes]
The fixed relational grammar of DGMM naturally supports structurally distinct families of admissible recall, defined by cue conditions. Four illustrative classes are:

\begin{itemize}

\item \textbf{Element-conditioned recall.}
$W_q$ is constructed by identifying Concept nodes sharing Element nodes with the cue representation. 
This supports associative retrieval grounded in shared perceptual or conceptual content.

\item \textbf{Temporally-scoped recall.}
$W_q$ is constrained to Concept nodes grounded within a specified time window relative to the cue. 
This supports episodic retrieval organized by temporal proximity.

\item \textbf{Source-conditioned recall.}
$W_q$ is restricted to Concept nodes associated with a specified Source node. 
This supports provenance-sensitive retrieval and enables multi-perspective analysis by selectively engaging memory from particular origins.

\item \textbf{Interaction-scoped recall.}
$W_q$ is restricted to Concept nodes associated with a specified Interaction node. 
This supports episode-sensitive retrieval and enables multi-interaction analysis by selectively engaging memory from particular interactions.

\end{itemize}
\end{remark}

These classes are not mutually exclusive. Composite recall strategies may condition simultaneously on element, temporal, and source structure, and alternative conditioning criteria derivable from the relational grammar, such as interaction-scoped recall bounded to a particular episode, are equally admissible. The classes above are offered as illustrations of the structural diversity admitted by the axioms and as a starting taxonomy for the companion paper's development of recall strategies.

\begin{remark}[Ordering over Recall Outcomes]
Let $\psi_1$ and $\psi_2$ be admissible recall operators in the sense of Definition~\ref{def:admissible-recall-operator}. For a fixed cue $q$ and memory state $M_t$, different admissible recall operators will in general return different subgraphs. A natural containment ordering exists over these outcomes:
\[
\psi_1 \leq \psi_2
\iff
\psi_1(q, M_t) \subseteq \psi_2(q, M_t)
\quad \text{for all } q \text{ and } t.
\]

This ordering admits a minimal recall operator returning only the most directly cue-relevant structure and approaches a maximal bound as recall expands to all nodes reachable from the cue within $M_t$. Working memory can be understood as any recalled subgraph situated within this ordering, with the boundedness axiom of Definition~\ref{def:admissible-recall-operator} ensuring that the maximal bound is never reached in practice. The lattice-theoretic properties of this ordering, including conditions under which minimal and maximal operators are well-defined, are developed in the recall companion paper.
\end{remark}

\subsubsection{Semantic Projection Analysis}
\label{sec:semntic-projection-analysis}
The operations described in this section belong to the analysis regime established in Definition~\ref{def:memory-operation-regimes}: they take $W_q$ as input, produce derived representations, and leave both $M_t$ and $W_q$ unchanged. 
Semantic meaning in DGMM is thus constructed by projecting the working memory subgraph into a semantic space:
\begin{equation}
Z_q = \Phi(W_q)
\end{equation}
where $\Phi$ is an embedding operator applied exclusively to $W_q$. This ensures that semantic representations are inherently cue- and context-dependent.

Analytic operators are then applied to $Z_q$ to derive higher-order properties, including similarity and relatedness, clustering for thematic coherence, importance based on structural contribution, and surprise or novelty as deviation from expected patterns. 
All analytic operations are read-only with respect to $M_t$ and $W_q$.

\begin{definition}[Admissible Embedding Operator]
\label{def:admissible-embedding-operator}
A function $\Phi_q : W_q \to Z_q$ is an admissible embedding operator under DGMM if and only if it satisfies the following constraints:

\begin{enumerate}[label=\textbf{E\arabic*:}, leftmargin=*]

\item \textbf{Post-recall domain.}
$\Phi_q$ takes $W_q$ as its sole input and has no direct access to $M_t$. Embeddings are computed exclusively from recalled structure.

\item \textbf{Read-only.}
$\Phi_q$ does not modify $M_t$ or $W_q$ during the embedding computation. The operation is strictly observational with respect to both long-term memory and working memory.

\item \textbf{Transience.}
The output $Z_q$ is transient and is not written back to $M_t$ unless explicitly committed through a subsequent ingestion operation.

\end{enumerate}
\end{definition}

These axioms define the space of compliant embedding functions rather than a unique function. 
The selection among admissible embedding operators, including whether to use graph neural networks, language model encoders, or other mechanisms, and the consequences of that selection for proposition generation and divergence measurement, are developed in the companion papers.

\subsubsection{Consolidation}
\label{sec:consolidation}
Consistent with the consolidation regime of Definition~\ref{def:memory-operation-regimes}, structural consolidation is a transformative operation governed by schema conformance and provenance preservation.  The following describes one admissible policy based on structural equivalence of Concept nodes.

Under the structural equivalence policy, Concept nodes are defined by their associated Element nodes. 
For a Concept node $v \in V_{\text{Concept}}$, let
\begin{equation}
E(v) = \{ e \in V_{\text{Elem}} \mid (v,e) \in E \}
\end{equation}
denote the set of Elements linked to $v$. 

Two Concept nodes $v_i$ and $v_j$ are structurally equivalent under this policy if:
\begin{equation}
E(v_i) = E(v_j).
\end{equation}

When structural equivalence is detected, the nodes $v_i$ and $v_j$ may be consolidated into a single Concept node $v^*$. 
Consolidation aggregates evidence rather than discarding it: Concept–Element relations are merged by summing their associated weights,
\begin{equation}
w(v^*, e) = w(v_i, e) + w(v_j, e) \quad \forall e \in E(v_i).
\end{equation}

Provenance preservation, as required by Definition~\ref{def:memory-operation-regimes}, is satisfied by reattaching all Time, Source, and Interaction nodes associated with $v_i$ and $v_j$ to the consolidated node $v^*$, unless redundant under existing uniqueness constraints. 
As a result, multiple temporal occurrences, provenance sources, and interaction contexts remain explicitly represented in the consolidated structure.

Consolidation is treated as an offline or background process, decoupled from cue-driven recall and working memory construction. 
Recall operations neither trigger consolidation nor depend on it. 
Alternative consolidation criteria such as approximate structural similarity, temporal proximity, or source-sensitive policies are equally admissible within the DGMM architecture.

All admissible consolidation policies, regardless of the equivalence criterion used, are governed by the schema conformance and provenance preservation constraints of Definition~\ref{def:memory-operation-regimes}. 
The development of admissible consolidation policies, their performance tradeoffs, and formal analysis of provenance preservation under structural reorganization are developed in the consolidation companion paper.

\section{Theoretical Capabilities and Evaluation Pathways}
\label{sec:theoretical-capabilities}
This section characterizes the theoretical capabilities afforded by the Dynamic Gist-Based Memory Model (DGMM) and identifies corresponding pathways for evaluation.  While algorithms, learning rules, and performance metrics are essential components of DGMM, the scope here is to establish which properties are well-defined under the DGMM representation and therefore amenable to principled analysis.

Throughout this section, a deliberate distinction is maintained between architectural capability and empirical validation.  The claims made here concern what DGMM makes observable by design and therefore any questions of optimization, efficiency, or task-level effectiveness are explicitly out of scope and deferred to future work.

\subsection{Structural Persistence and Recall Dynamics}
\label{sec:structural-persistence-recall}
DGMM represents memory as an explicit relational structure whose elements retain identity across time. 
Memory is not reduced to transient internal activations or implicit parameter states; instead, it is maintained as a persistent graph of concepts and relations that may be revisited, compared, and extended.

Recall in DGMM is modeled as a context-sensitive activation process over this structure. 
Given a cue $q$, recall induces a subgraph
\begin{equation}
R_t(q) \subseteq M_t,
\end{equation}
representing the subset of memory elements and relations that become relevant under that cue at time $t$. 
Because recall operates over explicit structure, the outcome of recall can be examined independently of downstream realization.

This formulation does not assume a fixed recall strategy. 
Different recall processes may engage different portions of memory under the same cue, and such variation is treated as an admissible property of the system rather than as an error condition. 

DGMM’s contribution at this level is to make the result of recall structurally explicit.

\subsection{Context Variability and Late Commitment}
\label{sec:context-variability-late-commitment}
A central design goal of DGMM is to support contextual variability without requiring premature semantic commitment.  Stable entities and relations persist in memory, while their interpretation and relevance remain contingent on the recall context.  As a result, the same underlying memory structure may support multiple perspectives or query intents without enforcing a single, globally fixed meaning.

From an evaluation standpoint, variation in recall across cues is therefore expected.  Differences in recalled structure do not necessarily indicate inconsistency or instability; they may instead reflect legitimate contextual shifts or alternative interpretive frames.  DGMM leverages cue-based contextual processing to reveal alternative interpretations, establishing a representation that supports the observation and analysis of interpretive variability.

\subsection{Surprise as Structural Deviation}
\label{sec:surrise-structural-deviation}
Surprise motivates learning and reevaluation by reflecting changes in knowledge acquisition and memory growth over time, revealing nuances, alternative associations, and emerging directions within the evolving memory structure. 
DGMM treats surprise as a property of structural deviation in recall instead of output-level anomaly, prediction error, or reward signal. 
Because memory is represented explicitly as a relational structure, changes in recall organization over time—such as the appearance of novel associations or the absence of previously typical patterns—can be identified at the level of memory engagement itself.

\begin{definition}[Cue-Conditioned Structural Surprise]
\label{def:cue-conditioned-structural-surprise}
Let $\Delta$ be a structural divergence operator satisfying the admissibility constraints of Definition~\ref{def:admissible-structural-divergence-operator}. 
Let $R_{t_1}(q) \subseteq M_{t_1}$ and $R_{t_2}(q) \subseteq M_{t_2}$ denote recall-induced subgraphs constructed under cue $q$ at times $t_1 < t_2$, respectively. 
Cue-conditioned structural surprise is defined as
\begin{equation}
S(q; t_1, t_2) = \Delta\bigl(R_{t_2}(q), R_{t_1}(q)\bigr).
\end{equation}
\end{definition}

Surprise is a forward-looking measure of structural divergence between recall-induced subgraphs under the same cue across time. 
It is defined relative to the history of recall under that cue, not relative to the global state of memory or to recall under different cues. 
The formal properties that any admissible $\Delta$ must satisfy are established in Definition~\ref{def:admissible-structural-divergence-operator}.

\begin{definition}[Admissible Structural Divergence Operator]
\label{def:admissible-structural-divergence-operator}
A function $\Delta$ is an admissible structural divergence operator for DGMM if it satisfies the following constraints:

\begin{enumerate}[label=\textbf{D\arabic*:}, leftmargin=*]

\item \textbf{Non-negativity.}
\[
\Delta\bigl(R_{t_1}(q), R_{t_2}(q)\bigr) \ge 0
\quad \text{for all admissible inputs.}
\]

\item \textbf{Zero self-divergence.}
\[
\Delta\bigl(R_{t_1}(q), R_{t_2}(q)\bigr) = 0
\quad \text{if } R_{t_1}(q) = R_{t_2}(q).
\]
Structural identity implies zero surprise.

\item \textbf{Domain restriction.}
$\Delta$ takes as input pairs of recall-induced subgraphs $R_{t_1}(q)$ and $R_{t_2}(q)$ such that
\[
R_{t_1}(q) \subseteq M_{t_1}, \quad
R_{t_2}(q) \subseteq M_{t_2}, \quad
t_1 < t_2,
\]
and both subgraphs are constructed under the same cue $q$. 
$\Delta$ is not defined over subgraphs constructed under different cues or over arbitrary graph pairs.

\item \textbf{Significance threshold.}
$\Delta$ is calibrated to a notion of structural significance appropriate to the memory context in which it operates. 
Trivial or incidental differences between recalled subgraphs need not produce positive divergence. 
The determination of what constitutes a significant structural shift is a topic for future work.

\end{enumerate}
\end{definition}

This definition makes no assumptions about how divergence is quantified or acted upon. 
It establishes only that surprise is well-defined as structural deviation between recall-induced subgraphs under the same cue across time.

\begin{remark}[Non-triviality of Locality]
By the additive growth invariant of the ingestion regime (Proposition~1), $M_{t_2}$ may be substantially larger than $M_{t_1}$, as a significant volume of new experience may have been ingested between $t_1$ and $t_2$. 
Without domain restriction (D3), a divergence operator defined over the full memory graph could be affected by this accumulated but unrecalled structure, producing surprise signals that reflect memory growth rather than genuine change in recall engagement. 

Proposition~3 closes this door: surprise is invariant under additions to $M_{t_2}$ that are not engaged by recall under $q$. 
A system ingesting continuously will therefore not generate surprise as a byproduct of memory accumulation. 
Surprise arises only when recall under a given cue engages memory differently than before. 
Locality establishes where surprise can arise—within recalled structure only. 
A companion architectural commitment governs how much structural change is required to register as surprise.
\end{remark}

Surprise in DGMM is not designed to respond to every structural difference between recalled subgraphs. 
It is designed to capture significant shifts, i.e., changes in recall organization that are meaningful with respect to the cue and the memory context. 
Calibrating $\Delta$ to significant structural shifts rather than fine-grained differences is an architectural commitment: a surprise signal that responds to every minor variation would not be a reliable indicator of genuine change in memory engagement. 
The formal characterization of what constitutes a significant shift, and the development of concrete divergence measures calibrated to different notions of significance, are developed in the surprise companion paper. 
This calibration is what makes surprise a structurally grounded signal of meaningful change rather than a noise-sensitive metric.

\subsubsection{Instantiations of the Structural Divergence Operator $\Delta$}
\label{sec:instantiations-structural-divergence-operator}

The definition of surprise above is intentionally abstract. 
DGMM commits to the existence of structural observables, not to a particular metric.

\paragraph{Local Structural Divergence.}
Let $N_k(v, t \mid q)$ denote the $k$-hop neighborhood of node $v$ within $R_t(q)$. 
Local divergence may be expressed as
\begin{equation}
\Delta_{\text{nbr}}(v; t_1, t_2 \mid q)
= 1 - \frac{\left| N_k(v, t_1 \mid q) \cap N_k(v, t_2 \mid q) \right|}
{\left| N_k(v, t_1 \mid q) \cup N_k(v, t_2 \mid q) \right|}.
\end{equation}

\paragraph{Embedding-Based Divergence (Post-Recall Interpretation).}
Embeddings in DGMM are computed after recall as query-conditioned projections of recalled structure. 
Let
\begin{equation}
\Phi_q : R_t(q) \rightarrow Z_t(q)
\end{equation}
produce transient embeddings $z_v(t \mid q)$. 
Representational divergence may be expressed as
\begin{equation}
\Delta_{\text{emb}}(v; t_1, t_2 \mid q)
= \left\lVert z_v(t_2 \mid q) - z_v(t_1 \mid q) \right\rVert.
\end{equation}

Embedding divergence reflects changes in how recalled structure is interpreted under a given query context. 
Such divergence may indicate structural reorganization in recall, shifts in contextual emphasis, or the emergence of new semantic nuance or meaning as interpretation evolves.

In DGMM, surprise is a direct manifestation of how recall structure deviates under comparable cues in a persistent memory substrate.

\subsection{Attribution, Source Sensitivity, and Structural Drift}
\label{sec:attribution-source-sensitivity-structural-drift}

Provenance is essential in memory-centric systems because it grounds recall in traceable contributions. 
DGMM represents provenance explicitly as part of memory structure. 
Sources are modeled as first-class nodes connected to the concepts to which each source contributes.

\begin{definition}[Provenance Observables]
\label{def:provenance-obervables}
Let $O \subseteq V_t$ denote the set of source nodes. 
Unless otherwise noted, this section considers the pre-consolidation regime, in which each concept node is associated with a single source node. 
This assumption simplifies interpretation but does not restrict generality.

For a recalled concept node $v \in R_t(q)$, the source attribution set is defined as
\begin{equation}
\mathrm{src}(v) = \{\, o \in O \mid (o,v) \in E_t \cap R_t(q) \,\}.
\end{equation}

Cue-conditioned source mass for source $o$ is defined as
\begin{equation}
m_o(q,t) = \sum_{v \in R_t(q)} w(v)\, \mathbf{1}[\, o \in \mathrm{src}(v) \,],
\end{equation}
where $\mathbf{1}[\cdot]$ is the indicator function.

The normalized source distribution is defined as
\begin{equation}
p_o(q,t) = \frac{m_o(q,t)}{\sum_{o' \in O} m_{o'}(q,t)}.
\end{equation}
\end{definition}

Changes in these distributions across time indicate shifts in provenance composition under a fixed cue. 
Such changes may arise from memory evolution, contextual variation, or recall strategy changes, making the resulting DGMM structure observable.

\subsection{Interpretability, Auditability, and Governance-Oriented Observables}
\label{sec:interpretability}

In this work, interpretability is grounded in the structural inspectability of recalled memory, reflecting explicit representations of memory elements, their relationships, and their engagement during recall. 
DGMM supports interpretability and auditability as consequences of explicit memory structure. 
For a cue $q$ at time $t$, the recalled subgraph
\begin{equation}
R_t(q)
\end{equation}
constitutes an inspectable artifact identifying which concepts and relations were engaged.

Interpretability here is structural rather than procedural. 
DGMM provides access to the structural footprint of recall, allowing inspection and comparison of recalled memory across time:
\begin{equation}
D_R(q; t_1, t_2) = \Delta\bigl(R_{t_1}(q), R_{t_2}(q)\bigr).
\end{equation}

Differences in recalled structure may reflect changes in memory, context, or recall strategy. 
While different recall strategies may yield different recalled subgraphs under the same cue, DGMM treats such variation as admissible; the recalled structures remain comparable as artifacts of memory engagement, independent of the processes that produced them. 
Regardless of cause, the recalled subgraphs remain comparable as structural objects.

Governance constraints may be expressed as predicates over recalled structure. 
Let $\Pi$ denote a set of such constraints. 
The constraint satisfaction signature is
\begin{equation}
\sigma_{\Pi}(q, t) = \{\, \pi \in \Pi \mid R_t(q) \models \pi \,\}.
\end{equation}

This supports inspection and auditing of recall outcomes without assumptions about execution mechanics.

\subsection{Proposition and Gist Generation}
\label{sec:proposition-generation}

Propositions are essential for thought simulation, ideation, and natural conversational interaction, enabling agents to articulate and explore interpretations without committing them to persistent memory. 
DGMM generates propositions, or gist-level interpretations, dynamically at recall time. 
Propositions are not stored as persistent memory entities; they are constructed by interpreting recalled structure in context.

Let
\begin{equation}
W_t(q) = \Psi\bigl(R_t(q)\bigr)
\end{equation}
denote a recall-scoped abstraction of salient relational patterns, and let $Z_t(q)$ denote post-recall embeddings. 
Propositions are generated as
\begin{equation}
P_t(q) = \rho\bigl(W_t(q), Z_t(q)\bigr).
\end{equation}

Because propositions are cue- and context-dependent, semantic evolution in DGMM occurs at the level of interpretation rather than storage. 
Different cues or perspectives may yield different propositions from the same underlying memory without requiring modification of stored structure.

\subsection{Evaluation Pathways}
\label{sec:evaluation}

These capabilities are not introduced as optional features but arise necessarily from DGMM’s representational commitments: once memory is explicit, persistent, and recall-conditioned, structural deviation, provenance-sensitive drift, and context-dependent interpretation become unavoidable and therefore analyzable properties of the system.

The capabilities described in this section define what DGMM is architecturally designed to enable. 
They suggest natural directions for empirical evaluation (e.g., analysis of recall structure, provenance composition, or interpretive stability) but do not themselves constitute empirical claims. 
The design and execution of such evaluations are deferred to future work.

\subsection{Architectural Invariants}
\label{sec:invariants}

This section formalizes a set of architectural invariants that follow directly from DGMM's commitments as established in Definition~\ref{def:memory-operation-regimes} (Memory Operation Regimes): the additive constraint on ingestion, the schema and provenance constraints on consolidation, the read-only constraint on recall, and the post-recall constraint on analysis.

\subsubsection{Structural Persistence under Additive Memory Growth}
\label{sec:structural-persistence}

DGMM represents long-term memory as a persistent relational structure that grows additively over time through the process of ingestion. 
Memory formation introduces new nodes and relations but does not overwrite, retype, or remove existing structures. 
This commitment yields an immediate structural invariant.

\begin{proposition}[Episodic Persistence under Additive Growth]
Let $M_t = (V_t, E_t)$ denote the DGMM memory graph at time $t$. 
If memory ingestion growth is additive, then for any $t_1 < t_2$,
\begin{equation}
V_{t_1} \subseteq V_{t_2}
\quad \text{and} \quad
E_{t_1} \subseteq E_{t_2}.
\end{equation}

Consequently, the subgraph induced by any interaction ingested at time $t_1$ remains structurally intact at all later times.
\end{proposition}

\begin{proof}
By architectural constraint, memory formation in DGMM appends new nodes and relations that conform to the fixed relational grammar while leaving existing structures unchanged. 
Additionally, Memory ingestions does not permit deletion or modification of prior nodes or edges. 
Therefore, memory ingestion growth is monotonic with respect to both nodes and relations, and all previously stored episodic subgraphs persist.
\end{proof}

This proposition formalizes a key architectural distinction between DGMM and systems in which memory is encoded implicitly in mutable parameters or periodically re-optimized structures. 
In DGMM, loss of access to prior experience can only occur through recall selection, not through structural erasure.

\subsubsection{Provenance Preservation under Consolidation}
\label{sec:provenance-preservation}
Consolidation reorganizes memory structure without adding new experience from outside the system. 
Unlike ingestion, which extends the graph additively, consolidation may merge, restructure, or re-weight existing nodes. 
This raises a question that ingestion does not: when existing structure is reorganized, is anything lost? 
For memory that is to be trusted as a substrate for attribution analysis, governance, and provenance-sensitive recall, the answer must be no. 
DGMM's provenance preservation constraint, established in Definition~\ref{def:memory-operation-regimes}, ensures that consolidation is structurally transformative but informationally conservative; that is, the episodic record of where knowledge came from, when it was acquired, and in what context it arose is never silently discarded, regardless of how the Concept node structure above it is reorganized. 
The one qualified exception concerns Time nodes, which may be generalized under the conditions stated in Proposition~2.

\begin{proposition}[Provenance Preservation under Consolidation]
Let $v_i, v_j \in V_{\text{Concept}}$ be Concept nodes consolidated into $v^*$ under any admissible consolidation policy. 
Let $P(v)$ denote the set of Time, Source, and Interaction nodes associated with a Concept node $v$. 
Then for any admissible consolidation:

\paragraph{(i) Source and Interaction Preservation.}
All Source and Interaction nodes associated with $v_i$ and $v_j$ are preserved and remain accessible through $v^*$:
\begin{equation}
\{\, o \in P(v_i) \cup P(v_j) \mid o \in V_{\text{Source}} \cup V_{\text{Int}} \,\}
\subseteq P(v^*).
\end{equation}

\paragraph{(ii) Time Preservation or Admissible Generalization.}
Time nodes associated with $v_i$ and $v_j$ are either preserved or replaced by an admissible generalization. 
A Time node generalization is admissible if and only if the generalized node is temporally consistent with the original nodes it replaces; that is, the generalized temporal reference covers the period of the original nodes without misrepresenting when the associated experience occurred.
\end{proposition}

\begin{proof}
By Definition~\ref{def:memory-operation-regimes}, all consolidation operations must be provenance-preserving: Source and Interaction nodes associated with any consolidated structure must be retained and remain accessible in the reorganized graph without exception. 
This is satisfied by construction for any admissible policy. 

Time nodes are governed by the same requirement subject to the generalization allowance stated in Definition~\ref{def:memory-operation-regimes}; generalization is permissible only when the resulting temporal reference is consistent with the original. 
Any consolidation policy that eliminates Source or Interaction nodes, or that replaces Time nodes with temporally inconsistent generalizations, is inadmissible under Definition~\ref{def:memory-operation-regimes} and therefore outside the scope of this proposition.
\end{proof}

\subsubsection{Locality of Cue-Conditioned Structural Surprise}
\label{sec:locality-cue-conditioned-surprise}

DGMM defines surprise as deviation in recall structure under comparable cues; therefore, since recall operates by constructing a bounded subgraph of long-term memory, surprise is necessarily localized to the recalled structure.

\begin{proposition}[Locality of Structural Surprise]
Let $R_t(q) \subseteq M_t$ denote the recall-induced subgraph constructed under cue $q$ at time $t$. 
Cue-conditioned structural surprise is invariant under changes to the remainder of memory $M_t \setminus R_t(q)$. 
That is, additions to long-term memory that are not engaged by recall under $q$ do not contribute to surprise.
\end{proposition}

\begin{proof}
Structural surprise is defined as divergence among recall-induced subgraphs under the same cue. 
By the domain restriction constraint of Definition~\ref{def:admissible-structural-divergence-operator} (D3), $\Delta$ is defined exclusively over the pair $\bigl(R_{t_2}(q), R_{t_1}(q)\bigr)$, where both subgraphs are constructed under the same cue $q$. 
Its value depends only on the structure of these two subgraphs and on no other portion of $M_{t_1}$ or $M_{t_2}$. 

Any additions to $M_{t_2}$ that are not engaged by recall under $q$ at time $t_2$ leave $R_{t_2}(q)$ unchanged by the read-only constraint of Definition~\ref{def:admissible-recall-operator} (R5). 
Therefore,
\begin{equation}
\Delta\bigl(R_{t_2}(q), R_{t_1}(q)\bigr)
\end{equation}
is unchanged, and structural surprise,
\begin{equation}
S(q; t_1, t_2),
\end{equation}
is invariant under such additions.
\end{proof}

This result ensures that routine memory accumulation does not generate spurious surprise signals. 
Surprise arises only when recall under a given cue engages memory differently than before, rather than as a byproduct of unrelated memory growth.

\subsubsection{Illustrative Example (Non-Overwriting Contextual Recall)}
\label{sec:example}

The following example illustrates how DGMM supports contextual variability in recall while preserving structural persistence.

At time $t_1$, an interaction is ingested from Source~A (e.g., an internal project document), asserting that Project~X succeeded. 
The corresponding Concept node is linked to Elements representing Project~X and Success and is grounded in time $t_1$ and its interaction context.

At time $t_2 > t_1$, a second interaction is ingested from Source~B (e.g., an external analyst report), asserting that Project~X was delayed. 
This produces a distinct Concept node linked to Elements representing Project~X and Delay, grounded in time $t_2$ and its own interaction context. 
Both episodic structures coexist in memory; neither overwrites the other.

Consider two recall cues applied later at time $t_3$:

\begin{itemize}

\item \textbf{Cue:} ``What happened with Project~X?''

Recall constructs a subgraph that includes both Concept nodes, their associated sources, and temporal grounding. 
Interpretation over this recalled structure may yield a proposition such as:
\begin{quote}
Project~X initially succeeded but later experienced delays.
\end{quote}

\item \textbf{Cue:} ``Recent risks related to Project~X.''

Recall constructs a more selective subgraph emphasizing the later interaction, its source, and delay-related elements. 
Interpretation may yield:
\begin{quote}
Project~X presents recent delivery risks.
\end{quote}

\end{itemize}

In both cases, the underlying memory structure is unchanged. 
Differences in interpretation arise solely from cue-conditioned recall and post-recall evaluation. 
Provenance and temporal grounding remain explicit and inspectable in both recalled subgraphs.

\subsubsection{Implications}
\label{sec:implications}

Together, these results demonstrate that DGMM’s representational commitments give rise to formally well-defined properties, including persistence of episodic memory, locality of cue-conditioned surprise, and contextual variability in recall without structural modification. 
These invariants apply to memory representation and recall; downstream interpretive processes, such as proposition generation, are intentionally left unconstrained. 
Importantly, these properties are independent of specific algorithms or learning rules and follow directly from the architecture’s treatment of memory as an explicit, additive, and recall-conditioned substrate.

\section{Positioning Within AI Architectures}
\label{sec:positioning}

DGMM occupies a distinct position within contemporary AI architectures by treating memory as an explicit, persistent, and recall-conditioned substrate. 
This architectural stance directly underpins the capabilities described in Section~\ref{sec:theoretical-capabilities} and differentiates DGMM from approaches in which memory is either implicit, externalized, or largely static.

Large language models encode prior experience implicitly within learned parameters. 
This design enables broad generalization but limits the ability to inspect, compare, or localize changes in memory engagement over time. 
Consequently, deviations in behavior are typically inferred indirectly from outputs or confidence measures. 
In contrast, DGMM’s explicit memory representation enables recall-induced structure to be examined directly, making phenomena such as cue-conditioned surprise (Section~\ref{sec:surrise-structural-deviation}) observable as structural deviation rather than as output-level anomalies.

Retrieval-augmented systems partially address this opacity by introducing access to external data sources. 
However, retrieved content is typically assembled transiently at inference time and does not form a persistent, evolving memory substrate with temporal and provenance grounding. 
As a result, changes in recall behavior are difficult to interpret as either accumulation or drift. 
DGMM’s integration of memory as a first-class architectural component allows recall outcomes to be compared across time and context, supporting analysis of structural drift and provenance-sensitive change (Section~\ref{sec:attribution-source-sensitivity-structural-drift}) without attributing such changes to retraining or retrieval heuristics.

Traditional knowledge graph--based systems provide explicit structure but are commonly oriented toward stable factual representations and curated updates. 
While this supports consistency, it limits the expression of episodic change and contextual reinterpretation. 
Unlike traditional knowledge graph approaches, DGMM does not assume that stored structure corresponds to stable semantic commitments; meaning is constructed at recall time from episodic structure, and consolidation into stable abstractions is neither required nor assumed. 

DGMM extends explicit structure into an episodic and temporal regime, enabling recall to engage different portions of memory under different cues and allowing meaning to evolve at the level of interpretation rather than storage. 
This distinction is central to DGMM’s treatment of proposition and gist generation (Section~\ref{sec:proposition-generation}), where propositions are constructed dynamically from recalled structure rather than stored as fixed semantic commitments.

Among existing systems, \textit{HippoRAG} (Gutiérrez et al., 2024) and \textit{Zep} (Rasmussen et al., 2025) represent the closest architectural precedents to DGMM and warrant direct engagement. 
HippoRAG models hippocampal indexing theory using a knowledge graph with cue-conditioned graph traversal, and Zep implements a temporal knowledge graph architecture for agent memory with time-grounded episodic structure. 
Both share DGMM's commitment to graph-structured memory and temporal grounding. 

The distinctions, however, are architecturally significant. 
Both systems embed semantic representations at ingestion time, fixing interpretive standpoints that cannot be revisited from a different context without modifying stored structure. 
DGMM defers this commitment to post-recall, preserving the capacity for reinterpretation under different cues without modifying memory. 

Neither system provides a formal architectural specification, as they are engineering artifacts without axiomatic treatment, which means their structural properties cannot be verified or reasoned about independently of implementation. 
DGMM's regime partition, recall axioms, and architectural invariants establish precisely these guarantees, which is what makes it suitable as a foundation for a series of companion papers rather than as a standalone system. 

Finally, provenance in both systems is treated as metadata rather than as a first-class structural element participating in the relational grammar. 
DGMM's Source nodes are typed graph citizens that support multi-source coexistence, provenance-sensitive recall, and attribution analysis without reconciliation at storage time. 
These capabilities follow from the architecture rather than from retrieval heuristics.

Taken together, these distinctions position DGMM as an architecture designed to make memory engagement itself a subject of analysis. 
The capabilities described in Section~\ref{sec:surrise-structural-deviation}—structural surprise, drift, attribution, and proposition generation—are direct consequences of this design choice. 
DGMM does not seek to replace parametric or retrieval-based approaches; rather, it provides a complementary framework for systems in which persistence, temporal evolution, and structural inspectability of memory are primary concerns. 

For example, one could hold the memory structure fixed while varying recall cues over time to analyze structural drift independently of training dynamics, or examine how proposition generation changes under controlled perturbations of recall context.

\section{Discussion}
\label{sec:discussion}

DGMM reframes explainability, accountability, and continuity as properties of memory representation rather than as challenges to be addressed post hoc through model interpretation or output analysis. 
By treating memory as an explicit, persistent, and evolving structure, DGMM enables examination of how recall, interpretation, and attribution change over time without requiring access to internal execution procedures or retraining dynamics.

A central implication of this work is the separation of memory storage, semantic interpretation, and downstream realization as distinct architectural concerns. 
In DGMM, memory persists independently of how it is interpreted at recall time, and interpretation may vary across cues or contexts without requiring modification of stored structure. 
This separation contrasts with architectures in which meaning, inference, and memory are tightly coupled, making it difficult to distinguish between accumulation, reinterpretation, and drift.

The capabilities outlined in Section~\ref{sec:theoretical-capabilities}---structural surprise, provenance-sensitive drift, and cue-conditioned proposition generation---follow directly from these representational commitments. 
Importantly, these capabilities are defined in terms of structural observability rather than performance outcomes. 
DGMM provides a framework for analyzing memory-centric behavior in a principled and inspectable manner.

DGMM is best understood as a complementary architectural approach. 
It does not seek to replace parametric language models, retrieval-augmented systems, or knowledge graph technologies, but to offer an alternative substrate in which persistent memory, temporal evolution, and contextual interpretation are first-class concerns. 
As such, DGMM suggests a research program focused on memory-centric AI, in which questions of surprise, drift, attribution, and interpretability are studied as properties of explicit memory structure rather than inferred indirectly from outputs.

Finally, this work is intentionally limited to architectural theory. 
While the proposed framework suggests natural pathways for empirical investigation---such as examining how recall structure evolves, how propositions vary across context, or how governance constraints manifest structurally---these questions are deferred to future work. 
The present contribution is to articulate a coherent representational foundation upon which such investigations may be built.

\section{Conclusion}
\label{sec:conclusion}

This paper introduces the Dynamic Gist-Based Memory Model (DGMM) as a theoretical and architectural contribution to artificial intelligence. 
By treating memory as a first-class, explicit, and persistent substrate, DGMM reframes continuity, provenance, and contextual grounding as properties of memory structure rather than as post hoc analytical challenges.

The primary contribution of this work is the articulation of an architectural foundation in which memory engagement, interpretation, and evolution can be examined directly through recall-induced structure. 
The capabilities discussed follow from representational commitments, independent of task-specific mechanisms or training procedures.

This work is intentionally limited in scope. 
It does not claim empirical performance improvements, nor does it prescribe learning algorithms or realization strategies. 
Instead, DGMM is offered as a research framework for memory-centric AI, providing a coherent basis for future empirical, analytical, and comparative studies of how persistent memory structures support interpretation over time.

\begin{acks}
The author gratefully acknowledges her dissertation committee-Dr. Kevin Huggins (Chair), Dr. Lola Bautista, Dr. Maria Vaida, Dr. Akeisha Belgrave and Dr. Kayden Jordan-for their guidance, critical feedback, and support throughout the development of this research. Any remaining errors or interpretations are solely those of the author.
\end{acks}

%%
%% The next line prints the references.
\printbibliography

@article{umbach_time_2020,
	title = {Time cells in the human hippocampus and entorhinal cortex support episodic memory},
	volume = {117},
	issn = {0027-8424, 1091-6490},
	url = {https://pnas.org/doi/full/10.1073/pnas.2013250117},
	doi = {10.1073/pnas.2013250117},
	language = {en},
	number = {45},
	urldate = {2024-02-22},
	journal = {Proceedings of the National Academy of Sciences},
	author = {Umbach, Gray and Kantak, Pranish and Jacobs, Joshua and Kahana, Michael and Pfeiffer, Brad E. and Sperling, Michael and Lega, Bradley},
	month = nov,
	year = {2020},
	pages = {28463--28474},
}

@misc{sahoo_systematic_2024,
	title = {A {Systematic} {Survey} of {Prompt} {Engineering} in {Large} {Language} {Models}: {Techniques} and {Applications}},
	shorttitle = {A {Systematic} {Survey} of {Prompt} {Engineering} in {Large} {Language} {Models}},
	url = {http://arxiv.org/abs/2402.07927},
	urldate = {2024-02-23},
	publisher = {arXiv},
	author = {Sahoo, Pranab and Singh, Ayush Kumar and Saha, Sriparna and Jain, Vinija and Mondal, Samrat and Chadha, Aman},
	month = feb,
	year = {2024},
	note = {arXiv:2402.07927 [cs]},
}

@article{peng_knowledge_2023,
	title = {Knowledge {Graphs}: {Opportunities} and {Challenges}},
	volume = {56},
	issn = {0269-2821, 1573-7462},
	shorttitle = {Knowledge {Graphs}},
	url = {https://link.springer.com/10.1007/s10462-023-10465-9},
	doi = {10.1007/s10462-023-10465-9},
	language = {en},
	number = {11},
	urldate = {2024-02-23},
	journal = {Artificial Intelligence Review},
	author = {Peng, Ciyuan and Xia, Feng and Naseriparsa, Mehdi and Osborne, Francesco},
	month = nov,
	year = {2023},
	pages = {13071--13102},
}

@article{schonhaut_neural_2023,
	title = {A neural code for time and space in the human brain},
	volume = {42},
	issn = {22111247},
	url = {https://linkinghub.elsevier.com/retrieve/pii/S2211124723012500},
	doi = {10.1016/j.celrep.2023.113238},
	language = {en},
	number = {11},
	urldate = {2024-02-25},
	journal = {Cell Reports},
	author = {Schonhaut, Daniel R. and Aghajan, Zahra M. and Kahana, Michael J. and Fried, Itzhak},
	month = nov,
	year = {2023},
	pages = {113238},
}

@misc{nylund_time_2023,
	title = {Time is {Encoded} in the {Weights} of {Finetuned} {Language} {Models}},
	url = {http://arxiv.org/abs/2312.13401},
	language = {en},
	urldate = {2024-02-27},
	publisher = {arXiv},
	author = {Nylund, Kai and Gururangan, Suchin and Smith, Noah A.},
	month = dec,
	year = {2023},
	note = {arXiv:2312.13401 [cs]},
}

@article{dhingra_time-aware_2022,
	title = {Time-{Aware} {Language} {Models} as {Temporal} {Knowledge} {Bases}},
	volume = {10},
	issn = {2307-387X},
	url = {https://direct.mit.edu/tacl/article/doi/10.1162/tacl_a_00459/110012/Time-Aware-Language-Models-as-Temporal-Knowledge},
	doi = {10.1162/tacl_a_00459},
	language = {en},
	urldate = {2024-02-27},
	journal = {Transactions of the Association for Computational Linguistics},
	author = {Dhingra, Bhuwan and Cole, Jeremy R. and Eisenschlos, Julian Martin and Gillick, Daniel and Eisenstein, Jacob and Cohen, William W.},
	month = mar,
	year = {2022},
	pages = {257--273},
}

@inproceedings{luu_time_2022,
	address = {Seattle, United States},
	title = {Time {Waits} for {No} {One}! {Analysis} and {Challenges} of {Temporal} {Misalignment}},
	url = {https://aclanthology.org/2022.naacl-main.435},
	doi = {10.18653/v1/2022.naacl-main.435},
	language = {en},
	urldate = {2024-02-27},
	booktitle = {Proceedings of the 2022 {Conference} of the {North} {American} {Chapter} of the {Association} for {Computational} {Linguistics}: {Human} {Language} {Technologies}},
	publisher = {Association for Computational Linguistics},
	author = {Luu, Kelvin and Khashabi, Daniel and Gururangan, Suchin and Mandyam, Karishma and Smith, Noah},
	year = {2022},
	pages = {5944--5958},
}

@article{fayyaz_model_2022,
	title = {A {Model} of {Semantic} {Completion} in {Generative} {Episodic} {Memory}},
	volume = {34},
	issn = {0899-7667, 1530-888X},
	url = {https://direct.mit.edu/neco/article/34/9/1841/112383/A-Model-of-Semantic-Completion-in-Generative},
	doi = {10.1162/neco_a_01520},
	language = {en},
	number = {9},
	urldate = {2024-02-29},
	journal = {Neural Computation},
	author = {Fayyaz, Zahra and Altamimi, Aya and Zoellner, Carina and Klein, Nicole and Wolf, Oliver T. and Cheng, Sen and Wiskott, Laurenz},
	month = aug,
	year = {2022},
	pages = {1841--1870},
}

@book{jones_models_2015,
	title = {Models of {Semantic} {Memory}},
	volume = {1},
	url = {https://academic.oup.com/edited-volume/41261/chapter/350846616},
	doi = {10.1093/oxfordhb/9780199957996.013.11},
	language = {en},
	urldate = {2024-03-01},
	publisher = {Oxford University Press},
	author = {Jones, Michael N. and Willits, Jon and Dennis, Simon},
	editor = {Busemeyer, Jerome R. and Wang, Zheng and Townsend, James T. and Eidels, Ami},
	month = dec,
	year = {2015},
}

@misc{zhao_survey_2023,
	title = {A {Survey} of {Large} {Language} {Models}},
	url = {http://arxiv.org/abs/2303.18223},
	language = {en},
	urldate = {2024-03-01},
	publisher = {arXiv},
	author = {Zhao, Wayne Xin and Zhou, Kun and Li, Junyi and Tang, Tianyi and Wang, Xiaolei and Hou, Yupeng and Min, Yingqian and Zhang, Beichen and Zhang, Junjie and Dong, Zican and Du, Yifan and Yang, Chen and Chen, Yushuo and Chen, Zhipeng and Jiang, Jinhao and Ren, Ruiyang and Li, Yifan and Tang, Xinyu and Liu, Zikang and Liu, Peiyu and Nie, Jian-Yun and Wen, Ji-Rong},
	month = nov,
	year = {2023},
	note = {arXiv:2303.18223 [cs]},
}

@book{hebb_donald_1949,
	title = {Donald {Hebb}: {The} {Organization} of {Behavior}},
	publisher = {John Wiley \& Sons},
	author = {Hebb, D.O.},
	editor = {Aertsen, A.},
	year = {1949},
}

@article{poo_what_2016,
	title = {What is memory? {The} present state of the engram},
	volume = {14},
	issn = {1741-7007},
	shorttitle = {What is memory?},
	url = {http://bmcbiol.biomedcentral.com/articles/10.1186/s12915-016-0261-6},
	doi = {10.1186/s12915-016-0261-6},
	language = {en},
	number = {1},
	urldate = {2024-03-09},
	journal = {BMC Biology},
	author = {Poo, Mu-ming and Pignatelli, Michele and Ryan, Tomás J. and Tonegawa, Susumu and Bonhoeffer, Tobias and Martin, Kelsey C. and Rudenko, Andrii and Tsai, Li-Huei and Tsien, Richard W. and Fishell, Gord and Mullins, Caitlin and Gonçalves, J. Tiago and Shtrahman, Matthew and Johnston, Stephen T. and Gage, Fred H. and Dan, Yang and Long, John and Buzsáki, György and Stevens, Charles},
	month = dec,
	year = {2016},
	pages = {40},
}

@misc{xu_hallucination_2024,
	title = {Hallucination is {Inevitable}: {An} {Innate} {Limitation} of {Large} {Language} {Models}},
	shorttitle = {Hallucination is {Inevitable}},
	url = {http://arxiv.org/abs/2401.11817},
	language = {en},
	urldate = {2024-03-17},
	publisher = {arXiv},
	author = {Xu, Ziwei and Jain, Sanjay and Kankanhalli, Mohan},
	month = jan,
	year = {2024},
	note = {arXiv:2401.11817 [cs]},
}

@misc{zhao_set_2024,
	title = {Set the {Clock}: {Temporal} {Alignment} of {Pretrained} {Language} {Models}},
	shorttitle = {Set the {Clock}},
	url = {http://arxiv.org/abs/2402.16797},
	urldate = {2024-03-17},
	publisher = {arXiv},
	author = {Zhao, Bowen and Brumbaugh, Zander and Wang, Yizhong and Hajishirzi, Hannaneh and Smith, Noah A.},
	month = feb,
	year = {2024},
	note = {arXiv:2402.16797 [cs]},
}

@misc{tonmoy_comprehensive_2024,
	title = {A {Comprehensive} {Survey} of {Hallucination} {Mitigation} {Techniques} in {Large} {Language} {Models}},
	url = {http://arxiv.org/abs/2401.01313},
	urldate = {2024-03-17},
	publisher = {arXiv},
	author = {Tonmoy, S. M. Towhidul Islam and Zaman, S. M. Mehedi and Jain, Vinija and Rani, Anku and Rawte, Vipula and Chadha, Aman and Das, Amitava},
	month = jan,
	year = {2024},
	note = {arXiv:2401.01313 [cs]},
}

@inproceedings{bender_dangers_2021,
	address = {Virtual Event Canada},
	title = {On the {Dangers} of {Stochastic} {Parrots}: {Can} {Language} {Models} {Be} {Too} {Big}? },
	isbn = {978-1-4503-8309-7},
	shorttitle = {On the {Dangers} of {Stochastic} {Parrots}},
	url = {https://dl.acm.org/doi/10.1145/3442188.3445922},
	doi = {10.1145/3442188.3445922},
	language = {en},
	urldate = {2024-03-17},
	booktitle = {Proceedings of the 2021 {ACM} {Conference} on {Fairness}, {Accountability}, and {Transparency}},
	publisher = {ACM},
	author = {Bender, Emily M. and Gebru, Timnit and McMillan-Major, Angelina and Shmitchell, Shmargaret},
	month = mar,
	year = {2021},
	pages = {610--623},
}

@inproceedings{lewis_retrieval-augmented_2020,
	title = {Retrieval-{Augmented} {Generation} for {Knowledge}-{Intensive} {NLP} {Tasks}},
	volume = {33},
	url = {https://proceedings.neurips.cc/paper_files/paper/2020/hash/6b493230205f780e1bc26945df7481e5-Abstract.html},
	urldate = {2024-03-24},
	booktitle = {Advances in {Neural} {Information} {Processing} {Systems}},
	publisher = {Curran Associates, Inc.},
	author = {Lewis, Patrick and Perez, Ethan and Piktus, Aleksandra and Petroni, Fabio and Karpukhin, Vladimir and Goyal, Naman and Küttler, Heinrich and Lewis, Mike and Yih, Wen-tau and Rocktäschel, Tim and Riedel, Sebastian and Kiela, Douwe},
	year = {2020},
	pages = {9459--9474},
}

@misc{wu_continual_2024,
	title = {Continual {Learning} for {Large} {Language} {Models}: {A} {Survey}},
	shorttitle = {Continual {Learning} for {Large} {Language} {Models}},
	url = {http://arxiv.org/abs/2402.01364},
	urldate = {2024-03-24},
	publisher = {arXiv},
	author = {Wu, Tongtong and Luo, Linhao and Li, Yuan-Fang and Pan, Shirui and Vu, Thuy-Trang and Haffari, Gholamreza},
	month = feb,
	year = {2024},
	note = {arXiv:2402.01364 [cs]},
}

@book{kant_critique_1781,
	address = {London},
	series = {Penguin classics},
	title = {Critique of pure reason},
	isbn = {978-0-14-044747-7},
	language = {eng},
	publisher = {Penguin Books},
	author = {Kant, Immanuel},
	editor = {Weigelt, Marcus},
	translator = {Müller, Friedrich Max},
	year = {1781},
}

@article{luo_local_2024,
	title = {Local {Interpretations} for {Explainable} {Natural} {Language} {Processing}: {A} {Survey}},
	issn = {0360-0300, 1557-7341},
	shorttitle = {Local {Interpretations} for {Explainable} {Natural} {Language} {Processing}},
	url = {http://arxiv.org/abs/2103.11072},
	doi = {10.1145/3649450},
	urldate = {2024-03-24},
	journal = {ACM Computing Surveys},
	author = {Luo, Siwen and Ivison, Hamish and Han, Caren and Poon, Josiah},
	month = mar,
	year = {2024},
	note = {arXiv:2103.11072 [cs]},
	pages = {3649450},
}

@inproceedings{ferrario_how_2022,
	address = {Seoul Republic of Korea},
	title = {How {Explainability} {Contributes} to {Trust} in {AI}},
	isbn = {978-1-4503-9352-2},
	url = {https://dl.acm.org/doi/10.1145/3531146.3533202},
	doi = {10.1145/3531146.3533202},
	language = {en},
	urldate = {2024-03-24},
	booktitle = {2022 {ACM} {Conference} on {Fairness}, {Accountability}, and {Transparency}},
	publisher = {ACM},
	author = {Ferrario, Andrea and Loi, Michele},
	month = jun,
	year = {2022},
	pages = {1457--1466},
}

@book{hegel_georg_1807,
	edition = {1},
	title = {Georg {Wilhelm} {Friedrich} {Hegel}: {\textless}{I}{\textgreater}{The} {Phenomenology} of {Spirit}{\textless}/{I}{\textgreater}},
	isbn = {978-1-139-05049-4 978-0-521-85579-2},
	shorttitle = {Georg {Wilhelm} {Friedrich} {Hegel}},
	url = {https://www.cambridge.org/core/product/identifier/9781139050494/type/book},
	doi = {10.1017/9781139050494},
	language = {en},
	urldate = {2024-03-25},
	publisher = {Cambridge University Press},
	author = {Hegel, Georg Wilhelm Fredrich},
	editor = {Pinkard, Terry and Baur, Michael},
	year = {1807},
}

@book{derrida_grammatology_1998,
	address = {Baltimore},
	edition = {Corrected ed},
	title = {Of grammatology},
	isbn = {978-0-8018-5830-7},
	language = {en},
	publisher = {Johns Hopkins University Press},
	author = {Derrida, Jacques},
	year = {1998},
}

@article{jiang_wikipedia-based_2017,
	title = {Wikipedia-based information content and semantic similarity computation},
	volume = {53},
	issn = {03064573},
	url = {https://linkinghub.elsevier.com/retrieve/pii/S0306457316303934},
	doi = {10.1016/j.ipm.2016.09.001},
	language = {en},
	number = {1},
	urldate = {2024-08-24},
	journal = {Information Processing \& Management},
	author = {Jiang, Yuncheng and Bai, Wen and Zhang, Xiaopei and Hu, Jiaojiao},
	month = jan,
	year = {2017},
	pages = {248--265},
}

@inproceedings{yao_tree_2023,
	title = {Tree of {Thoughts}: {Deliberate} {Problem} {Solving} with {Large} {Language} {Models}},
	language = {en},
	booktitle = {Advances in {Neural} {Information} {Processing} {Systems} 36},
	author = {Yao, Shunyu and Yu, Dian and Zhao, Jeffrey and Shafran, Izhak and Griffiths, Thomas L and Cao, Yuan and Narasimhan, Karthik},
	year = {2023},
}

@misc{das_larimar_2024,
	title = {Larimar: {Large} {Language} {Models} with {Episodic} {Memory} {Control}},
	shorttitle = {Larimar},
	url = {http://arxiv.org/abs/2403.11901},
	language = {en},
	urldate = {2024-08-29},
	publisher = {arXiv},
	author = {Das, Payel and Chaudhury, Subhajit and Nelson, Elliot and Melnyk, Igor and Swaminathan, Sarath and Dai, Sihui and Lozano, Aurélie and Kollias, Georgios and Chenthamarakshan, Vijil and Jiří and Navrátil and Dan, Soham and Chen, Pin-Yu},
	month = aug,
	year = {2024},
	note = {arXiv:2403.11901 [cs]},
}

@misc{anokhin_arigraph_2024,
	title = {{AriGraph}: {Learning} {Knowledge} {Graph} {World} {Models} with {Episodic} {Memory} for {LLM} {Agents}},
	shorttitle = {{AriGraph}},
	url = {http://arxiv.org/abs/2407.04363},
	language = {en},
	urldate = {2024-08-29},
	publisher = {arXiv},
	author = {Anokhin, Petr and Semenov, Nikita and Sorokin, Artyom and Evseev, Dmitry and Burtsev, Mikhail and Burnaev, Evgeny},
	month = jul,
	year = {2024},
	note = {arXiv:2407.04363 [cs]},
}

@misc{gutierrez_hipporag_2024,
	title = {{HippoRAG}: {Neurobiologically} {Inspired} {Long}-{Term} {Memory} for {Large} {Language} {Models}},
	shorttitle = {{HippoRAG}},
	url = {http://arxiv.org/abs/2405.14831},
	language = {en},
	urldate = {2024-08-29},
	publisher = {arXiv},
	author = {Gutiérrez, Bernal Jiménez and Shu, Yiheng and Gu, Yu and Yasunaga, Michihiro and Su, Yu},
	month = may,
	year = {2024},
	note = {arXiv:2405.14831 [cs]},
}

@article{biswas_knowledge_2023,
	title = {Knowledge {Graph} {Embeddings}: {Open} {Challenges} and {Opportunities}},
	volume = {1},
	copyright = {Creative Commons Attribution 4.0 International license, info:eu-repo/semantics/openAccess},
	issn = {2942-7517},
	shorttitle = {Knowledge {Graph} {Embeddings}},
	url = {https://drops.dagstuhl.de/entities/document/10.4230/TGDK.1.1.4},
	doi = {10.4230/TGDK.1.1.4},
	language = {en},
	number = {1},
	urldate = {2025-01-26},
	journal = {Transactions on Graph Data and Knowledge (TGDK)},
	publisher = {Schloss Dagstuhl – Leibniz-Zentrum für Informatik},
	author = {Biswas, Russa and Kaffee, Lucie-Aimée and Cochez, Michael and Dumbrava, Stefania and Jendal, Theis E. and Lissandrini, Matteo and Lopez, Vanessa and Mencía, Eneldo Loza and Paulheim, Heiko and Sack, Harald and Vakaj, Edlira Kalemi and de Melo, Gerard},
	collaborator = {Hogan, Aidan and Horrocks, Ian and Hotho, Andreas and Kagal, Lalana},
	year = {2023},
	note = {Artwork Size: 32 pages, 1553014 bytes
Medium: application/pdf},
	pages = {4:1--4:32},
}

@misc{liu_incremental_2022,
	title = {Incremental {Prompting}: {Episodic} {Memory} {Prompt} for {Lifelong} {Event} {Detection}},
	shorttitle = {Incremental {Prompting}},
	url = {http://arxiv.org/abs/2204.07275},
	doi = {10.48550/arXiv.2204.07275},
	language = {en},
	urldate = {2025-03-22},
	publisher = {arXiv},
	author = {Liu, Minqian and Chang, Shiyu and Huang, Lifu},
	month = sep,
	year = {2022},
	note = {arXiv:2204.07275 [cs]},
}

@misc{kong_dynamic_2024,
	title = {Dynamic {Semantic} {Memory} {Retention} in {Large} {Language} {Models}: {An} {Exploration} of {Spontaneous} {Retrieval} {Mechanisms}},
	copyright = {http://creativecommons.org/licenses/by-nc-nd/4.0/},
	shorttitle = {Dynamic {Semantic} {Memory} {Retention} in {Large} {Language} {Models}},
	url = {https://www.authorea.com/users/850341/articles/1237030-dynamic-semantic-memory-retention-in-large-language-models-an-exploration-of-spontaneous-retrieval-mechanisms?commit=53a3b49dda2b55ad0e99aae294e2e82539cee33b},
	doi = {10.22541/au.173040837.79423019/v1},
	language = {en},
	urldate = {2025-04-05},
	author = {Kong, Juri and Liang, Hong and Zhang, Yuan and Li, Hongxiang and Shen, Pengcheng and Lu, Fang},
	month = oct,
	year = {2024},
}

@misc{dakat_enhancing_2024,
	title = {Enhancing {Large} {Language} {Models} through {Dynamic} {Contextual} {Memory} {Embedding}: {A} {Technical} {Evaluation}},
	copyright = {https://creativecommons.org/licenses/by-nc-sa/4.0/},
	shorttitle = {Enhancing {Large} {Language} {Models} through {Dynamic} {Contextual} {Memory} {Embedding}},
	url = {https://www.techrxiv.org/users/845354/articles/1233973-enhancing-large-language-models-through-dynamic-contextual-memory-embedding-a-technical-evaluation?commit=f16319ce67f2d9a441f862efb817431689dd7778},
	doi = {10.36227/techrxiv.172978236.68956715/v1},
	language = {en},
	urldate = {2025-04-05},
	author = {Dakat, Igor and Langley, Isadora and Montgomery, Lysander and Bennett, Rosalin and Blackwood, Lysandra},
	month = oct,
	year = {2024},
}

@misc{bernar_exploring_2024,
	title = {Exploring the {Concept} of {Dynamic} {Memory} {Persistence} in {Large} {Language} {Models} for {Optimized} {Contextual} {Comprehension}},
	copyright = {http://creativecommons.org/licenses/by-nc-nd/4.0/},
	url = {https://www.authorea.com/users/852953/articles/1238675-exploring-the-concept-of-dynamic-memory-persistence-in-large-language-models-for-optimized-contextual-comprehension?commit=884a5060f1fdbfd87f90fc3a888f4c26fa4045cc},
	doi = {10.22541/au.173101368.88048313/v1},
	language = {en},
	urldate = {2025-04-05},
	author = {Bernar, Brian and Winters, Harrison and Fischer, Laurence and Meyer, Brandon and Gyllenborg, Mitchell},
	month = nov,
	year = {2024},
}

@misc{zhang_survey_2025,
	title = {A {Survey} of {Graph} {Retrieval}-{Augmented} {Generation} for {Customized} {Large} {Language} {Models}},
	url = {http://arxiv.org/abs/2501.13958},
	doi = {10.48550/arXiv.2501.13958},
	language = {en},
	urldate = {2025-04-06},
	publisher = {arXiv},
	author = {Zhang, Qinggang and Chen, Shengyuan and Bei, Yuanchen and Yuan, Zheng and Zhou, Huachi and Hong, Zijin and Dong, Junnan and Chen, Hao and Chang, Yi and Huang, Xiao},
	month = jan,
	year = {2025},
	note = {arXiv:2501.13958 [cs]},
}

@misc{brown_language_2020,
	title = {Language {Models} are {Few}-{Shot} {Learners}},
	url = {http://arxiv.org/abs/2005.14165},
	doi = {10.48550/arXiv.2005.14165},
	language = {en},
	urldate = {2025-04-06},
	publisher = {arXiv},
	author = {Brown, Tom B. and Mann, Benjamin and Ryder, Nick and Subbiah, Melanie and Kaplan, Jared and Dhariwal, Prafulla and Neelakantan, Arvind and Shyam, Pranav and Sastry, Girish and Askell, Amanda and Agarwal, Sandhini and Herbert-Voss, Ariel and Krueger, Gretchen and Henighan, Tom and Child, Rewon and Ramesh, Aditya and Ziegler, Daniel M. and Wu, Jeffrey and Winter, Clemens and Hesse, Christopher and Chen, Mark and Sigler, Eric and Litwin, Mateusz and Gray, Scott and Chess, Benjamin and Clark, Jack and Berner, Christopher and McCandlish, Sam and Radford, Alec and Sutskever, Ilya and Amodei, Dario},
	month = jul,
	year = {2020},
	note = {arXiv:2005.14165 [cs]},
}

@misc{gai_zhenhua_achieving_2024,
	title = {Achieving {Higher} {Factual} {Accuracy} in {Llama} {LLM} with {Weighted} {Distribution} of {Retrieval}-{Augmented} {Generation}},
	copyright = {https://creativecommons.org/licenses/by/4.0/legalcode},
	url = {https://osf.io/ctw8v},
	doi = {10.31219/osf.io/ctw8v},
	language = {en},
	urldate = {2025-04-20},
	author = {{Gai, Zhenhua} and Tong, Lianxin and Ge, Quan},
	month = may,
	year = {2024},
}

@misc{rasmussen_zep_2025,
	title = {Zep: {A} {Temporal} {Knowledge} {Graph} {Architecture} for {Agent} {Memory}},
	shorttitle = {Zep},
	url = {http://arxiv.org/abs/2501.13956},
	doi = {10.48550/arXiv.2501.13956},
	language = {en},
	urldate = {2025-08-10},
	publisher = {arXiv},
	author = {Rasmussen, Preston and Paliychuk, Pavlo and Beauvais, Travis and Ryan, Jack and Chalef, Daniel},
	month = jan,
	year = {2025},
	note = {arXiv:2501.13956 [cs]},
}

@misc{zhang_respecting_2025,
	title = {Respecting {Temporal}-{Causal} {Consistency}: {Entity}-{Event} {Knowledge} {Graphs} for {Retrieval}-{Augmented} {Generation}},
	shorttitle = {Respecting {Temporal}-{Causal} {Consistency}},
	url = {http://arxiv.org/abs/2506.05939},
	doi = {10.48550/arXiv.2506.05939},
	language = {en},
	urldate = {2025-10-11},
	publisher = {arXiv},
	author = {Zhang, Ze Yu and Li, Zitao and Li, Yaliang and Ding, Bolin and Low, Bryan Kian Hsiang},
	month = jun,
	year = {2025},
	note = {arXiv:2506.05939 [cs]},
}

@article{knez_event-centric_2023,
	title = {Event-{Centric} {Temporal} {Knowledge} {Graph} {Construction}: {A} {Survey}},
	volume = {11},
	issn = {2227-7390},
	shorttitle = {Event-{Centric} {Temporal} {Knowledge} {Graph} {Construction}},
	url = {https://www.mdpi.com/2227-7390/11/23/4852},
	doi = {10.3390/math11234852},
	language = {en},
	number = {23},
	urldate = {2025-10-11},
	journal = {Mathematics},
	author = {Knez, Timotej and Žitnik, Slavko},
	month = dec,
	year = {2023},
	pages = {4852},
}

@misc{cai_survey_2024,
	title = {A {Survey} on {Temporal} {Knowledge} {Graph}: {Representation} {Learning} and {Applications}},
	shorttitle = {A {Survey} on {Temporal} {Knowledge} {Graph}},
	url = {http://arxiv.org/abs/2403.04782},
	doi = {10.48550/arXiv.2403.04782},
	language = {en},
	urldate = {2025-10-11},
	publisher = {arXiv},
	author = {Cai, Li and Mao, Xin and Zhou, Yuhao and Long, Zhaoguang and Wu, Changxu and Lan, Man},
	month = mar,
	year = {2024},
	note = {arXiv:2403.04782 [cs]},
}

@inproceedings{saxena_question_2021,
	address = {Online},
	title = {Question {Answering} {Over} {Temporal} {Knowledge} {Graphs}},
	url = {https://aclanthology.org/2021.acl-long.520},
	doi = {10.18653/v1/2021.acl-long.520},
	language = {en},
	urldate = {2025-10-11},
	booktitle = {Proceedings of the 59th {Annual} {Meeting} of the {Association} for {Computational} {Linguistics} and the 11th {International} {Joint} {Conference} on {Natural} {Language} {Processing} ({Volume} 1: {Long} {Papers})},
	publisher = {Association for Computational Linguistics},
	author = {Saxena, Apoorv and Chakrabarti, Soumen and Talukdar, Partha},
	year = {2021},
	pages = {6663--6676},
}

@misc{sun_timelinekgqa_2025,
	title = {{TimelineKGQA}: {A} {Comprehensive} {Question}-{Answer} {Pair} {Generator} for {Temporal} {Knowledge} {Graphs}},
	shorttitle = {{TimelineKGQA}},
	url = {http://arxiv.org/abs/2501.04343},
	doi = {10.48550/arXiv.2501.04343},
	language = {en},
	urldate = {2025-10-11},
	publisher = {arXiv},
	author = {Sun, Qiang and Li, Sirui and Huynh, Du and Reynolds, Mark and Liu, Wei},
	month = jan,
	year = {2025},
	note = {arXiv:2501.04343 [cs]},
}

@article{sadeghian_chronor_2021,
	title = {{ChronoR}: {Rotation} {Based} {Temporal} {Knowledge} {Graph} {Embedding}},
	volume = {35},
	issn = {2374-3468, 2159-5399},
	shorttitle = {{ChronoR}},
	url = {https://ojs.aaai.org/index.php/AAAI/article/view/16802},
	doi = {10.1609/aaai.v35i7.16802},
	language = {en},
	number = {7},
	urldate = {2025-10-11},
	journal = {Proceedings of the AAAI Conference on Artificial Intelligence},
	author = {Sadeghian, Ali and Armandpour, Mohammadreza and Colas, Anthony and Wang, Daisy Zhe},
	month = may,
	year = {2021},
	pages = {6471--6479},
}

@misc{yang_beyond_2025,
	title = {Beyond {Single} {Pass}, {Looping} {Through} {Time}: {KG}-{IRAG} with {Iterative} {Knowledge} {Retrieval}},
	shorttitle = {Beyond {Single} {Pass}, {Looping} {Through} {Time}},
	url = {http://arxiv.org/abs/2503.14234},
	doi = {10.48550/arXiv.2503.14234},
	language = {en},
	urldate = {2025-10-11},
	publisher = {arXiv},
	author = {Yang, Ruiyi and Xue, Hao and Razzak, Imran and Hacid, Hakim and Salim, Flora D.},
	month = may,
	year = {2025},
	note = {arXiv:2503.14234 [cs]},
}

@article{bullmore_economy_2012,
	title = {The economy of brain network organization},
	volume = {13},
	copyright = {http://www.springer.com/tdm},
	issn = {1471-003X, 1471-0048},
	url = {https://www.nature.com/articles/nrn3214},
	doi = {10.1038/nrn3214},
	language = {en},
	number = {5},
	urldate = {2025-10-12},
	journal = {Nature Reviews Neuroscience},
	author = {Bullmore, Ed and Sporns, Olaf},
	month = may,
	year = {2012},
	pages = {336--349},
}

@inproceedings{duan_few-shot_2023,
	title = {Few-shot {Generation} via {Recalling} {Brain}-{Inspired} {Episodic}-{Semantic} {Memory}},
	url = {https://openreview.net/forum?id=dxPcdEeQk9},
	language = {en},
	urldate = {2025-11-13},
	author = {Duan, Zhibin and Zhiyi, Lv and Wang, Chaojie and Chen, Bo and An, Bo and Zhou, Mingyuan},
	month = nov,
	year = {2023},
}

@book{foucault_archaeology_1972,
	address = {Westminster},
	title = {The {Archaeology} of {Knowledge}},
	isbn = {978-0-394-71106-5 978-0-307-81925-3},
	language = {eng},
	publisher = {Knopf Doubleday Publishing Group},
	author = {Foucault, Michel},
	year = {1972},
}

@inproceedings{miao_episodic_2024,
	address = {Bangkok, Thailand},
	title = {Episodic {Memory} {Retrieval} from {LLMs}: {A} {Neuromorphic} {Mechanism} to {Generate} {Commonsense} {Counterfactuals} for {Relation} {Extraction}},
	shorttitle = {Episodic {Memory} {Retrieval} from {LLMs}},
	url = {https://aclanthology.org/2024.findings-acl.146/},
	doi = {10.18653/v1/2024.findings-acl.146},
	urldate = {2025-11-14},
	booktitle = {Findings of the {Association} for {Computational} {Linguistics}: {ACL} 2024},
	publisher = {Association for Computational Linguistics},
	author = {Miao, Xin and Li, Yongqi and Zhou, Shen and Qian, Tieyun},
	editor = {Ku, Lun-Wei and Martins, Andre and Srikumar, Vivek},
	month = aug,
	year = {2024},
	pages = {2489--2511},
}

@article{kumar_semantic_2021,
	title = {Semantic memory: {A} review of methods, models, and current challenges},
	volume = {28},
	issn = {1069-9384, 1531-5320},
	shorttitle = {Semantic memory},
	url = {https://link.springer.com/10.3758/s13423-020-01792-x},
	doi = {10.3758/s13423-020-01792-x},
	language = {en},
	number = {1},
	urldate = {2025-11-25},
	journal = {Psychonomic Bulletin \& Review},
	author = {Kumar, Abhilasha A.},
	month = feb,
	year = {2021},
	pages = {40--80},
}

@misc{nagy_interplay_2025,
	title = {Interplay of episodic and semantic memory arises from adaptive compression},
	url = {https://osf.io/emky9_v1},
	doi = {10.31234/osf.io/emky9},
	language = {en},
	urldate = {2025-11-28},
	publisher = {PsyArXiv},
	author = {Nagy, David G. and Orban, Gergo and Wu, Charley M},
	month = jan,
	year = {2025},
}

@misc{dalessandro_genesis_2025,
	title = {{GENESIS}: {A} {Generative} {Model} of {Episodic}-{Semantic} {Interaction}},
	shorttitle = {{GENESIS}},
	url = {http://arxiv.org/abs/2510.15828},
	doi = {10.48550/arXiv.2510.15828},
	language = {en},
	urldate = {2025-11-28},
	publisher = {arXiv},
	author = {D'Alessandro, Marco and D'Amato, Leo and Elkano, Mikel and Uriz, Mikel and Pezzulo, Giovanni},
	month = oct,
	year = {2025},
	note = {arXiv:2510.15828 [q-bio]},
}

@article{spens_generative_2024,
	title = {A generative model of memory construction and consolidation},
	volume = {8},
	issn = {2397-3374},
	url = {https://www.nature.com/articles/s41562-023-01799-z},
	doi = {10.1038/s41562-023-01799-z},
	language = {en},
	number = {3},
	urldate = {2025-11-28},
	journal = {Nature Human Behaviour},
	author = {Spens, Eleanor and Burgess, Neil},
	month = jan,
	year = {2024},
	pages = {526--543},
}

@article{howard_distributed_2002,
	title = {A {Distributed} {Representation} of {Temporal} {Context}},
	volume = {46},
	copyright = {https://www.elsevier.com/tdm/userlicense/1.0/},
	issn = {00222496},
	url = {https://linkinghub.elsevier.com/retrieve/pii/S0022249601913884},
	doi = {10.1006/jmps.2001.1388},
	language = {en},
	number = {3},
	urldate = {2025-11-29},
	journal = {Journal of Mathematical Psychology},
	author = {Howard, Marc W. and Kahana, Michael J.},
	month = jun,
	year = {2002},
	pages = {269--299},
}

@misc{spens_hippocampo-neocortical_2025,
	title = {Hippocampo-neocortical interaction as compressive retrieval-augmented generation},
	language = {en},
	publisher = {bioRxiv},
	author = {Spens, Eleanor and Burgess, Neil},
	year = {2025},
}

@article{kirkpatrick_overcoming_2017,
	title = {Overcoming catastrophic forgetting in neural networks},
	volume = {114},
	issn = {0027-8424, 1091-6490},
	url = {https://pnas.org/doi/full/10.1073/pnas.1611835114},
	doi = {10.1073/pnas.1611835114},
	language = {en},
	number = {13},
	urldate = {2026-04-23},
	journal = {Proceedings of the National Academy of Sciences},
	author = {Kirkpatrick, James and Pascanu, Razvan and Rabinowitz, Neil and Veness, Joel and Desjardins, Guillaume and Rusu, Andrei A. and Milan, Kieran and Quan, John and Ramalho, Tiago and Grabska-Barwinska, Agnieszka and Hassabis, Demis and Clopath, Claudia and Kumaran, Dharshan and Hadsell, Raia},
	month = mar,
	year = {2017},
	pages = {3521--3526},
}

@article{lange_continual_2021,
	title = {A continual learning survey: {Defying} forgetting in classification tasks},
	issn = {0162-8828, 2160-9292, 1939-3539},
	shorttitle = {A continual learning survey},
	url = {http://arxiv.org/abs/1909.08383},
	doi = {10.1109/TPAMI.2021.3057446},
	language = {en},
	urldate = {2026-04-23},
	journal = {IEEE Transactions on Pattern Analysis and Machine Intelligence},
	author = {Lange, Matthias De and Aljundi, Rahaf and Masana, Marc and Parisot, Sarah and Jia, Xu and Leonardis, Ales and Slabaugh, Gregory and Tuytelaars, Tinne},
	year = {2021},
	note = {arXiv:1909.08383 [cs]},
	pages = {1--1},
}

@incollection{thrun_child_1998,
	address = {Boston, MA},
	title = {Child: {A} {First} {Step} {Towards} {Continual} {Learning}},
	isbn = {978-1-4613-7527-2 978-1-4615-5529-2},
	shorttitle = {Child},
	url = {http://link.springer.com/10.1007/978-1-4615-5529-2_11},
	doi = {10.1007/978-1-4615-5529-2_11},
	language = {en},
	urldate = {2026-04-23},
	booktitle = {Learning to {Learn}},
	publisher = {Springer US},
	author = {Ring, Mark B.},
	editor = {Thrun, Sebastian and Pratt, Lorien},
	year = {1998},
	pages = {261--292},
}

@article{thrun_lifelong_1995,
	title = {Lifelong robot learning},
	volume = {15},
	copyright = {https://www.elsevier.com/tdm/userlicense/1.0/},
	issn = {09218890},
	url = {https://linkinghub.elsevier.com/retrieve/pii/092188909500004Y},
	doi = {10.1016/0921-8890(95)00004-Y},
	language = {en},
	number = {1-2},
	urldate = {2026-04-23},
	journal = {Robotics and Autonomous Systems},
	author = {Thrun, Sebastian and Mitchell, Tom M.},
	month = jul,
	year = {1995},
	pages = {25--46},
}

@article{van_de_ven_three_2022,
	title = {Three types of incremental learning},
	volume = {4},
	issn = {2522-5839},
	url = {https://www.nature.com/articles/s42256-022-00568-3},
	doi = {10.1038/s42256-022-00568-3},
	language = {en},
	number = {12},
	urldate = {2026-04-23},
	journal = {Nature Machine Intelligence},
	author = {Van De Ven, Gido M. and Tuytelaars, Tinne and Tolias, Andreas S.},
	month = dec,
	year = {2022},
	pages = {1185--1197},
}

@misc{wang_comprehensive_2024,
	title = {A {Comprehensive} {Survey} of {Continual} {Learning}: {Theory}, {Method} and {Application}},
	shorttitle = {A {Comprehensive} {Survey} of {Continual} {Learning}},
	url = {http://arxiv.org/abs/2302.00487},
	doi = {10.48550/arXiv.2302.00487},
	language = {en},
	urldate = {2026-04-23},
	publisher = {arXiv},
	author = {Wang, Liyuan and Zhang, Xingxing and Su, Hang and Zhu, Jun},
	month = feb,
	year = {2024},
	note = {arXiv:2302.00487 [cs]},
}

\end{document}